\documentclass{article}
\usepackage{ijcai19}

\usepackage{times}
\usepackage{soul}
\usepackage{url}
\usepackage[hidelinks]{hyperref}
\usepackage[utf8]{inputenc}
\usepackage[small]{caption}
\usepackage{graphicx}
\usepackage{amsmath}
\usepackage{booktabs}
\usepackage{algorithm}
\usepackage{algorithmic}
\usepackage{amsmath}\allowdisplaybreaks
\usepackage{xspace,amssymb,epsfig}
\usepackage{dsfont}
\usepackage{algorithm}
\usepackage{algorithmic}

\usepackage{ifthen}

\newcommand{\RR}{\mathbb{R}}

\newcommand{\bC}{\boldsymbol{C}}
\newcommand{\bD}{\boldsymbol{D}}
\newcommand{\bI}{\boldsymbol{I}}

\newcommand{\bR}{\boldsymbol{R}}
\newcommand{\bS}{\boldsymbol{S}}
\newcommand{\bT}{\boldsymbol{T}}
\newcommand{\bU}{\boldsymbol{U}}

\newcommand{\bb}{\boldsymbol{b}}
\newcommand{\bi}{\boldsymbol{i}}
\newcommand{\bj}{\boldsymbol{j}}
\newcommand{\bp}{\boldsymbol{p}}

\newcommand{\bx}{\boldsymbol{x}}
\newcommand{\by}{\boldsymbol{y}}
\newcommand{\bvarepsilon}{\boldsymbol{\varepsilon}}

\newcommand{\btheta}{\boldsymbol{\theta}}

\newcommand{\cH}{\boldsymbol{\mathcal{H}}}

\newcommand{\argmax}{\mathrm{argmax}}

\newcommand{\norm}[1]{\left\| #1 \right\|}
\newcommand{\abs}[1]{\left| #1 \right|}

\newtheorem{theorem}{Theorem}

\newtheorem{lemma}[theorem]{Lemma}
\newenvironment{claim}[1]{\par\noindent\underline{Claim:}\space#1}{}
\newenvironment{proof}{\noindent {\textbf{Proof. }}}{$\Box$ \medskip}

\usepackage[capitalize,noabbrev]{cleveref}
\usepackage{tikz}
\usetikzlibrary{shapes,shapes.misc,positioning}
\usetikzlibrary{patterns,arrows,decorations.pathreplacing}
\newcommand{\EE}[1]{\mathbb{E} \left[#1\right]}
\newcommand{\EEt}[1]{\mathbb{E}_t \left[#1\right]}
\newcommand{\PP}[1]{\mathbb{P}\left[#1\right]}
\newcommand{\ceil}[1]{\left\lceil #1 \right\rceil}
\newcommand{\ip}[1]{\langle #1 \rangle}

\newcommand{\linucbone}{{\tt LinUCB-One}\xspace}
\newcommand{\linucbind}{{\tt LinUCB-Ind}\xspace}
\newcommand{\club}{{\tt CLUB}\xspace}
\newcommand{\sclub}{{\tt SCLUB}\xspace}
\newcommand{\citet}[1]{\citeauthor{#1} \shortcite{#1}}

\newif\ifsup\supfalse
\suptrue

\newcommand{\compilehidecomments}{false}

\ifthenelse{ \equal{\compilehidecomments}{true} }{%
	\newcommand{\wei}[1]{}
	\newcommand{\qinshi}[1]{}
}{
	\newcommand{\wei}[1]{{\color{blue!50!black}  [\text{Wei:} #1]}}
	
}

\title{Improved Algorithm on Online Clustering of Bandits}

\author{
Shuai Li$^1$
\and
Wei Chen$^2$\and
Shuai Li$^3$\And
Kwong-Sak Leung$^1$
\affiliations
$^1$The Chinese University of Hong Kong\\
$^2$Microsoft\\
$^3$University of Cambridge
\emails
\{shuaili, ksleung\}@cse.cuhk.edu.hk,
weic@microsoft.com,
shuaili.sli@gmail.com
}

\begin{document}

\maketitle

\begin{abstract}
We generalize the setting of online clustering of bandits by allowing non-uniform distribution over user frequencies. A more efficient algorithm is proposed with simple set structures to represent clusters. We prove a regret bound for the new algorithm which is free of the minimal frequency over users. The experiments on both synthetic and real datasets consistently show the advantage of the new algorithm over existing methods.
\end{abstract}

\section{Introduction}

The problem of stochastic multi-armed bandit (MAB) \cite{bubeck2012regret,LS18book} has been widely studied in the field of statistics and machine learning, where the learner selects an action each round with the goal of maximizing the cumulative rewards over all rounds (or equivalently minimize the cumulative regret). 
One successful application of MAB algorithms is recommendation systems \cite{aggarwal2016recommender}, such as recommendations of movies, restaurants and news articles. For the setting where new users and items constantly appear where simultaneous exploration and exploitation are naturally needed, bandit algorithms can cope with the ``cold start'' problem and improve itself as time goes on.

To adjust bandit algorithms for large-scale applications, structural assumptions are usually added on actions and reward functions, among which linear structure is one of the most common due to its simplicity and effectiveness. Each item in the stochastic linear bandits \cite{abe1999associative,rusmevichientong2010linearly,dani2008stochastic,abbasi2011improved} is represented by a feature vector whose expected reward is an unknown linear function on the feature vector. In many systems certain side information about users and items are employed as the feature vectors to guarantee recommendation quality \cite{li2010contextual,chu2011contextual}.

Usually user features are not accurate enough or even do not exist to characterize user preferences. When user features exist, these static features are usually combined together with item features to be actions in linear bandit models. This processing might have implacable modeling bias and also neglect the often useful tool of collaborative filtering. One way to utilize the collaborative effect of users is to discover their clustering structure. Clustering methods have been proved to be important in supervised/unsupervised recommendations \cite{linden2003amazon,woo2014cluster}. The works \cite{gentile2014online,li2016collaborative} start to bring these useful techniques to bandit algorithms. These online clustering of bandit algorithms adaptively learn the clustering structure over users based on the collaborative recommendation results to gather information of users' similarity and still keep certain diversity across users.

There are two main challenges in the existing works on online clustering of bandits \cite{gentile2014online,li2016collaborative,korda2016distributed,gentile2017context,li2016art}. (a) The existing works assume users come to be served in a uniform manner. However, as stated by long tail effect and $80/20$ rule, the imbalance between users always exist especially with a large number of low-frequency users. The low-frequency users deteriorate the performance of existing algorithms. (b) The existing works use graphs to represent users, use edges to denote the similarity between users and use connected components to represent clusters. 
As a result, the learner may split two dissimilar users into two clusters only when it cuts every path between these two users. Even when the edge (representing the similarity) between two users are cut, the two users might still need to stay in the same cluster for a long time. Therefore these algorithms are not efficient in identifying the underlying clusters.

To overcome the first problem, we include frequency properties into the criterion of an underlying clustering structure. To overcome the second problem, we split a user out of her current cluster once we find the current cluster contains other users dissimilar to this one. Since the split operation might be a bit radical and premature, we add a new merge operation to reconcile the radical effect of the split operation. Based on these two building operations, complex graph structure can be discarded and replaced by simple set structures where each cluster is represented by a set of users.

This paper makes four major contributions:
\begin{enumerate}
\item We generalize the setting of online clustering of bandits to allow non-uniform frequency distributions over users.

\item We design both split and merge operations on clusters together with set representations of clusters to accelerate the process of identifying underlying clusters.

\item We analyze the regret of our new algorithm to get a regret bound free of the term $1/p_{\min}$, where $p_{\min}$ is the minimal frequency probability among all users and the existing algorithms could not avoid it.

\item We compare our algorithm to the existing algorithms on both synthetic and real datasets. A series of experiments consistently show the advantage of our algorithm over the existing ones.
\end{enumerate}

\subsection{Related work} 

The most related works are a series of studies on online clustering of bandits. \citet{gentile2014online} first establish the setting of online clustering of bandits and first show the effectiveness of using graph structure to represent clustering over users. A follow-up \cite{li2016collaborative} considers the collaborative filtering effects on both users and items but only under the setting of static item pool. They maintain a clustering over items and for each item cluster there is a clustering over all users. Their analysis is only built on a special case that the items form an orthonormal basis in the feature space. Based on \cite{gentile2014online}, \citet{korda2016distributed} design a distributed version of confidence ball algorithms in peer to peer networks with limited communications where the peers in the same cluster solve the same bandit problem. \citet{gentile2017context} present a context-aware clustering bandit algorithm where the clustering structure depends on items, though with a high computational cost. All these works assume the user frequency distribution is uniform and use connected components of graphs to represent clusters. 

Besides these works with new algorithms and theoretical guarantees, there are many progress of applying these online algorithms in real applications. \citet{nguyen2014dynamic} design a variant of clustering of bandit algorithms by performing k-means clustering method \cite{macqueen1967some} on the estimates of user weight vectors with a known number of clusters. 
\citet{christakopoulou2018learning} present a variant of clustering of bandit algorithms based on Thompson sampling and allow recommending a list of items to users.
\citet{kwon2018overlapping} allows the underlying clustering has overlapping.
All these works focus on experimental performance and do not include theoretical guarantees, especially on regret bounds.
\citet{li2018online} study an application of using clustering of bandit algorithms to improve the performance of recommending a list of items under a specific click model, and they provide a regret guarantee.

There is a group of works on stochastic low-rank bandits \cite{katariya2017stochastic,katariya2017bernoulli,kveton2017stochastic}. They assume the reward matrix of users and items are of low-rank and the learner selects a user-item pair each time. The goal of the learner is to find the user-item pair that has the maximum reward. This is different from the setting of online clustering of bandits where there is no control over user appearances and the goal of the learner is to maximize satisfaction of each appearing user.

The works on linear bandits are cornerstones of clustering of bandits, where the latter has different weight vectors and adaptively clusters users. It is first studied by \cite{abe1999associative} and refined by \cite{dani2008stochastic,abbasi2011improved} and many others. \citet{turgay2018multi} present a similar work to make use of collaborative results but they assume the similarity information is known.

\section{Model and Problem Definitions}
\label{sec:setting}

In this section, we formulate the setting of ``online clustering of bandits''. There are $n_u$ users, denoted by the set $[n_u] = \{1, \ldots, n_u\}$, with a fixed but {\em unknown} distribution $p = (p_1, \ldots, p_{n_u}) \in \Delta_{n_u}$ over users. Here $\Delta_{n} = \{ x \in [0,1]^n: \sum_{i=1}^n x_i = 1 \}$ denotes the simplex in $\RR^n$. There is also a fixed but {\em unknown} distribution $\rho$ over $\{x \in \RR^d: \norm{x} \le 1\}$, which is the set of feature vectors for all items, and we simply call such a feature vector an item. At each time $t$, a user $\bi_t \in [n_u]$ is randomly drawn from the distribution $p$, and $L$ items are drawn independently from distribution $\rho$ to form a feasible item set $\bD_t$. The learning agent receives the user index $\bi_t$ and the feasible item set $\bD_t$, selects an item $\bx_{t} \in \bD_t$ and recommends it to the user. After the user checks the item $\bx_t$, the learning agent receives a reward $\by_t$. The setting follows the previous work \cite{gentile2014online} but allows non-uniform distribution over users.
	
Let $\cH_t = \{\bi_1, \bx_1, \by_1, \ldots, \bi_{t-1}, \bx_{t-1}, \by_{t-1}, \bi_t\}$ be all the information after receiving the user index in time $t$. Then the action $\bx_{t}$ is $\cH_{t}$-adaptive. Henceforth, we will write $\EEt{\cdot}$ for $\EE{\cdot \mid \cH_{t}}$ for the sake of notational convenience, use the boldface symbols to denote random variables, and denote $[n]$ to be the set $\{1,\ldots,n\}$.
	
For any time $t$, given a fixed user $i$ at time $t$ and a fixed item $x \in \bD_t$ selected for user $i$, we define the random reward of item $x$ for user $i$ to be $\by_t(i,x) = \theta_i^\top x + \bvarepsilon_{i, t, x}$, where (a) $\theta_{i}$ is a fixed but {\em unknown} weight vector in $\RR^{d \times 1}$ with $\norm{\theta_{i}} \le 1$, independently of other items and other user behaviors; and (b) $\bvarepsilon_{i,t,x}$ is an $\cH_t$-conditional random noise with zero mean and $R$-sub-Gaussian tail, i.e. $\EE{\exp(\nu\bvarepsilon_{i,t,x})} \le \exp(R^2 \nu^2/2)$ for every $\nu \in \RR$. Then the mean reward is $\EE{\by_t(i,x)} = \theta_i^\top x$. At time $t$, the learning agent recommends $\bx_t$ to user $\bi_t$ and receives reward $\by_t(\bi_t, \bx_t)$, which satisfies
\begin{align*}
    \EEt{\by_t(\bi_t, \bx_t) \mid \bi_t, \bx_t} = \EE{\by_t(\bi_t, \bx_t) \mid \cH_t, \bi_t, \bx_t} = \theta_{\bi_t}^{\top} \bx_t\,.
\end{align*}

Suppose there are $m$ (unknown) different weight vectors in the set $\{\theta_i :  i \in [n_u]\}$. An underlying clustering exists over the users according to the weight vectors, where the users with the same weight vector form a cluster. The goal is to learn the underlying clustering to help the recommendation process. We make the following assumptions on the weight vectors $\{\theta_i : i \in [n_u]\}$, the distribution $p$ over users, and the distribution $\rho$ on items.

	
\paragraph{Gap between weight vectors.} For any two different weight vectors $\theta_{i_1} \neq \theta_{i_2}$, there is a (fixed but unknown) gap $\gamma$ between them: $\norm{\theta_{i_1} - \theta_{i_2}} \ge \gamma > 0$.

\paragraph{Item regularity.} For item distribution $\rho$, $\mathbb{E}_{\bx \sim \rho}[\bx \bx^{\top}]$ is full rank with minimal eigenvalue $\lambda_x > 0$. Also at all time $t$, for any fixed unit vector $\theta \in \RR^d$, $(\theta^{\top} \bx)^2$ has sub-Gaussian tail with variance parameter $\sigma^2 \le \lambda_x^2 /(8 \log(4 L))$. We assume a lower bound for $\lambda_x$ is known.
	
\paragraph{Gap between user frequencies.} Users with same weight vectors will have same frequencies: $\theta_{i_1} = \theta_{i_2}$ implies that $p_{i_1} = p_{i_2}$. Also for any two different user frequency probabilities $p_{i_1} \neq p_{i_2}$, there is a (fixed but unknown) gap $\gamma_p$ between them: $\abs{p_{i_1} - p_{i_2}} \ge \gamma_p > 0$.

All the assumptions except the last one follow the previous work \cite{gentile2014online}. The first part of last assumption could be relaxed. Discussions are provided in a later section.
	
The optimal item for user $\bi_t$ in round $t$ is $x_{\bi_t, \bD_t}^\ast = \argmax_{x \in \bD_t} \theta_{\bi_t}^\top x$. Then the expected regret for user $\bi_t$ in time $t$ is $\bR_t = \theta_{\bi_t}^\top x_{\bi_t, \bD_t}^\ast - \theta_{\bi_t}^\top \bx_{t}$. The goal of the learning agent is to minimize the expected cumulative regret 
\begin{align*}
    R(T) = \EE{\sum_{t=1}^{T} \bR_t} = \EE{\sum_{t=1}^{T} \left( \theta_{\bi_t}^\top x_{\bi_t, \bD_t}^\ast - \theta_{\bi_t}^\top \bx_{t} \right) }\,,
\end{align*}
where the expectation is taken over the randomness of users $\bi_1, \ldots, \bi_T$, the randomness of the item sets $\bD_1, \ldots, \bD_T$ and the possible randomness in selected items $\bx_1, \ldots, \bx_T$.
	
\section{Algorithm}
\label{sec:algorithms}

\begin{algorithm}[t]
\caption{\sclub}
\label{alg:sclub}
\begin{algorithmic}[1]
\STATE \textbf{Input}: exploration parameter $\alpha_{\theta}, \alpha_p > 0$, and $\beta > 0$;
\STATE \label{algoline: user initialization} Initialize information for each user $i \in [n_u]$ by $S_i = \bI_{d \times d}, b_i = \mathbf{0}_{d \times 1}, T_i = 0$;
\STATE \label{algoline: cluster initialization} Initialize the set of cluster indexes by $J = \{1\}$ and initialize the single cluster by $S^1 = \bI_{d \times d}, b^1 = \mathbf{0}_{d \times 1}, T^1 = 0, C^1 = [n_u], j(i) = 1,\forall i$;
\FOR{\label{algoline: for s} $s =1, 2, \dots$}
\STATE \label{algoline: mark users unchecked} Mark every user unchecked for each cluster;
\STATE \label{algoline: set pivot} For each cluster $j$, compute $\tilde{\bT}^j = \bT^j$, 
and $\tilde{\btheta}^j = (\bS^j)^{-1} \bb^j$;
\FOR{\label{algoline: for t} $t = 1, \ldots, 2^s$}
\STATE \label{algoline: compute total time step tau} Compute the total time step $\tau = 2^s - 2 + t$;
\STATE \label{algoline: receive user and item set Dt} Receive a user $\bi_{\tau}$ and an item set $\bD_{\tau} \subset \RR^d$;
\STATE \label{algoline: get cluster index and its associated information} Get the cluster index $\bj = \bj(\bi_\tau)$ and its associated information $(\bS^{\bj}, \bb^{\bj}, \bT^{\bj})$;
\STATE \label{algoline: recommend item to the user with largest U} Recommend item $\bx_{\tau} = \argmax_{x \in \bD_{\tau}} \bU(x)$ where
\begin{align*}
\bU(x) = (\bb^{\bj})^\top (\bS^{\bj})^{-1} x + \beta \norm{x}_{(\bS^{\bj})^{-1}}
\end{align*}
to user $\bi_{\tau}$ and receive feedback $\by_{\tau}$
\STATE Run \textbf{Update}
\STATE Run \textbf{Split}
\STATE Mark user $\bi_\tau$ has been checked;
\STATE Run \textbf{Merge}
\ENDFOR{~~$t$}
\ENDFOR{~~$s$}
\end{algorithmic}
\end{algorithm}

\begin{algorithm}[t]
\caption{Update}
\label{alg:update}
\begin{algorithmic}
\STATE Update the information for user $\bi_{\tau}$ and cluster $\bj$
\begin{align*}
&\bS_{\bi_{\tau}} = \bS_{\bi_\tau} + \bx_{\tau} \bx_{\tau}^{\top},  \quad \bb_{\bi_\tau} = \bb_{\bi_\tau} + \by_{\tau} \bx_{\tau}, \\
&\bT_{\bi_\tau} = \bT_{\bi_\tau} + 1, \quad \hat{\bp}_{\bi_{\tau}} = \bT_{\bi_\tau} / \tau, \\
&\hat{\btheta}_{\bi_\tau} = \bS_{\bi_{\tau}}^{-1} \bb_{\bi_\tau},\\
&\bS^{\bj} = \bS^{\bj} + \bx_{\tau} \bx_{\tau}^{\top},  \quad \bb^{\bj} = \bb^{\bj} + \by_{\tau} \bx_{\tau}, \\
&\bT^{\bj} = \bT^{\bj} + 1, \quad \hat{\bp}_{i'} = \bT_{i'} / \tau, ~~\forall i' \in \bC^{\bj},\\
&\hat{\btheta}^{\bj} = (\bS^{\bj})^{-1} \bb^{\bj}\,.
\end{align*}		
\end{algorithmic}
\end{algorithm}

\begin{algorithm}[t]
\caption{Split}
\label{alg:split}
\begin{algorithmic}
\STATE $F(T) = \sqrt{\frac{1 + \ln(1 + T)}{1 + T}}$;
\IF{ \label{algoline: check split} $\norm{\hat{\btheta}_{\bi_\tau} - \tilde{\btheta}^{\bj}} > \alpha_{\theta} \left(F(\bT_{\bi_\tau}) + F(\tilde{\bT}^{\bj})\right)$ or there exists user $i' \in \bC^{\bj}$ with $\abs{\hat{\bp}_{\bi_\tau} - \hat{\bp}_{i'}} > 2 \alpha_p F(\tau)$}
\STATE /\ /\ Split user $\bi_\tau$ from cluster $\bj$ and form a new cluster $j' (= \max J + 1)$ of user $\bi_\tau$
\begin{align*}
&\bS^{\bj} = \bS^{\bj} - \bS_{\bi_\tau} + \bI, \quad \bb^{\bj} = \bb^{\bj} - \bb_{\bi_\tau}, \\
&\bT^{\bj} = \bT^{\bj} - \bT_{\bi_{\tau}}, \quad \bC^{\bj} = \bC^{\bj} - \{\bi_\tau\};\\
&\bS^{j'} = \bS_{\bi_{\tau}}, \quad \bb^{j'} = \bb_{\bi_\tau}, \quad \bT^{j'} = \bT_{\bi_{\tau}}, \quad \bC^{j'} = \{\bi_\tau\}
\end{align*}
\ENDIF
\end{algorithmic}
\end{algorithm}

\begin{algorithm}[t]
\caption{Merge}
\label{alg:merge}
\begin{algorithmic}
\FOR{ any two checked clusters $j_1 < j_2$ satisfying
\begin{align*}
\norm{\hat{\btheta}^{j_1} - \hat{\btheta}^{j_2}} < \frac{\alpha_{\theta}}{2} \left(F(\bT^{j_1}) + F(\bT^{j_2})\right)
\end{align*}
and
\begin{align*}
\abs{\hat{\bp}^{j_1} - \hat{\bp}^{j_2}} < \alpha_p F(\tau) 
\end{align*}
}
\STATE /\ /\ Merge them: add information of cluster $j_2$ to cluster $j_1$ and remove cluster $j_2$
\begin{align*}
&\bS^{j_1} = \bS^{j_1} + \bS^{j_2} - \bI\,, \quad \bb^{j_1} = \bb^{j_1} + \bb^{j_2}\,, \\
&\bT^{j_1} = \bT^{j_1} + \bT^{j_2}\,, \quad \bC^{j_1} = \bC^{j_1} \bigcup \bC^{j_2}\,; \\
\vspace{-0.2in}
&j(i) = j_1 ~~\forall i \in \bC^{j_2} \text{ and delete cluster } j_2
\end{align*}
\ENDFOR{}
\end{algorithmic}
\end{algorithm}

\ifsup
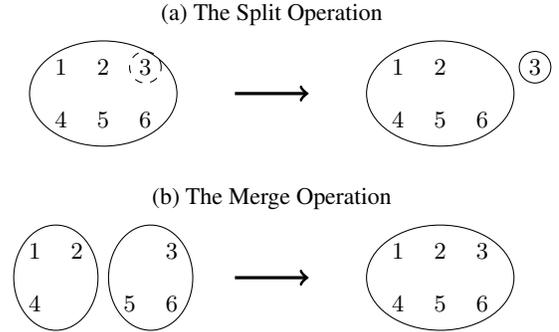
\begin{figure}[thb!]
\centering
\begin{tikzpicture}[scale=0.7,font=\small]

\node at (0,0) {(a) The Split Operation};
\draw[->, very thick] (-0.7,-1.5) -- (0.7,-1.5);
\node at (-4, -1) {$1$};
\node at (-3.2, -1) {$2$};
\node at (-2.4, -1) {$3$};
\node at (-4, -2) {$4$};
\node at (-3.2, -2) {$5$};
\node at (-2.4, -2) {$6$};
\draw (-3.2,-1.5) ellipse (1.4cm and 1cm);
\draw[dashed] (-2.4,-1) circle (0.3cm);

\node at (2.4, -1) {$1$};
\node at (3.2, -1) {$2$};
\node at (2.4, -2) {$4$};
\node at (3.2, -2) {$5$};
\node at (4, -2) {$6$};
\draw (3.2,-1.5) ellipse (1.4cm and 1cm);
\node at (5, -1) {$3$};
\draw (5,-1) circle (0.3cm);

\node at (0,-3.5) {(b) The Merge Operation};
\draw[->, very thick] (-0.7,-5) -- (0.7,-5);
\node at (-4.5, -4.5) {$1$};
\node at (-3.7, -4.5) {$2$};
\node at (-1.9, -4.5) {$3$};
\node at (-4.5, -5.5) {$4$};
\node at (-2.7, -5.5) {$5$};
\node at (-1.9, -5.5) {$6$};
\draw (-4.1,-5) ellipse (0.8cm and 1cm);
\draw (-2.3,-5) ellipse (0.8cm and 1cm);

\node at (2.4, -4.5) {$1$};
\node at (3.2, -4.5) {$2$};
\node at (4, -4.5) {$3$};
\node at (2.4, -5.5) {$4$};
\node at (3.2, -5.5) {$5$};
\node at (4, -5.5) {$6$};
\draw (3.2,-5) ellipse (1.4cm and 1cm);
\end{tikzpicture}
\caption{The illustrations of the split and merge operations on sets.}
\label{fig:split merge}
\end{figure}
\fi

In this section, we introduce our algorithm of ``set-based clustering of bandits (\sclub)'' to deal with the online clustering of bandit problem. Recall that the related previous work \cite{gentile2014online,li2016collaborative} adopted connected components in graphs to represent clusters and the learning process only split clusters. In contrast, our algorithm uses sets to represent clusters and allow both split and merge operations on sets in the learning process. 
\ifsup
\cref{fig:split merge} provides an illustration of the split and merge operations on sets. The learner will split a user out of the current set if it finds inconsistency between the user and the cluster (\cref{fig:split merge}a); the learner will merge two existing clusters if they are close enough (\cref{fig:split merge}b). The meanings of inconsistency and similarity will be explained later. 
\fi
The pseudocode is provided in \cref{alg:sclub}.

In the algorithm, we store a profile of $(\bS_i, \bb_i, \bT_i)$ for each user $i$, where $\bT_i$ is the number of times that the user $i$ has appeared, $\bS_i$ is the Gramian matrix and $\bb_i$ is the moment vector of regressand by regressors. For example, if $(x_1, y_1), \ldots, (x_n, y_n)$ are pairs of items and corresponding rewards collected for user $i$ before time $t$, then 
\begin{align*}
    T_i = n\,, \quad S_i = \sum_{k=1}^{n} x_k x_k^\top\,, \quad b_i = \sum_{k=1}^{n} x_k y_k\,.
\end{align*}
The user profiles are initialized in the beginning of the algorithm (line \ref{algoline: user initialization}), and the clustering is initialized to be a single set with index $1$ containing all users (line \ref{algoline: cluster initialization}). We store a profile of $(\bS^j, \bb^j, \bT^j, \bC^j)$ for each cluster $j$, where $\bC^j$ is the set of all user indexes in this cluster and
\begin{align*}
\bS^j = \bI+\sum_{i \in \bC^j} (\bS_i - \bI), \quad \bb^j = \sum_{i \in \bC^j} \bb_i, \quad \bT^j = \sum_{i \in \bC^j} \bT_i
\end{align*}
are the aggregate information in the cluster $j$. As a convention, we use subscript $\bS_i, \bb_i, \bT_i$ to denote the information related to a user $i$ and use superscript $\bS^j, \bb^j, \bT^j$ to denote the information related to a cluster $j$.

The learning agent proceeds in phases (line \ref{algoline: for s}), where the $s$-th phase contains $2^s$ rounds (line \ref{algoline: for t}). Each phase has an accuracy level such that the learner could put accurate clusters aside and focus on exploring inaccurate clusters. In the beginning of each phase, mark every user to be unchecked (line \ref{algoline: mark users unchecked}), which is used to identify good clusters in the current accuracy level. If all users in a cluster are checked, then this cluster is {\it checked} to be an accurate cluster in the current phase. Compute the estimated weight vectors for each cluster (line \ref{algoline: set pivot}), which serve as the pivots to be compared for splitting its containing users. Note that the pivot weight vectors are denoted as $\tilde{\btheta}^j$ instead of $\hat{\btheta}^j$.

At each time $t$ in the $s$-th phase, a user $\bi_{\tau}$ comes to be served with candidate item set $\bD_{\tau}$ (line \ref{algoline: receive user and item set Dt}), where $\tau$ denotes the index of the total time step (line \ref{algoline: compute total time step tau}). 
First obtain the information $(\bS^{\bj}, \bb^{\bj}, \bT^{\bj})$ for the cluster $\bj$ of user $\bi_\tau$ (line \ref{algoline: get cluster index and its associated information}) and then recommend based on these information (line \ref{algoline: recommend item to the user with largest U}). Here $\norm{x}_A = \sqrt{x^\top A x}$ for any positive definite matrix $A$.
After receiving feedback $\by_\tau$, the learner will {\it update} information (Algorithm \ref{alg:update}) and check if a split is possible (Algorithm \ref{alg:split}). 

By the assumptions in \cref{sec:setting}, users in one cluster have same weight vector and same frequency. Thus if a cluster is a good cluster, which only contains users of same underlying cluster, the estimated parameters for containing users will be close to estimated parameters of the cluster. We call a user $i$ is \textit{consistent} with the current cluster $j$ if $\hat{\theta}_i$ is close to $\hat{\theta}^j$ and $\hat{p}_i$ is close to $\hat{p}_{i'}$ for any other user $i'$ of cluster $j$. If a user is inconsistent with the current cluster, the learner will split her out (\cref{alg:split}).

We call two checked clusters \textit{consistent} if their estimated weight vectors and estimated frequencies are close enough. The algorithm will merge two checked clusters if they are consistent (Algorithm \ref{alg:merge}). Here $\hat{p}^j = T^j / (\abs{C^j} \tau)$ is the average frequency of cluster $j$. The condition of clusters to be \textit{checked} is to avoid merging two clusters that are not accurate enough in the current level. The operations of split and merge, together with the conditions, can guarantee that the clusters are good at the end of each phase in the corresponding accuracy level with high probability (the analysis is part of the regret bound proof). 

\section{Results}

Recall that $d$ is the feature dimension, $m$ is the (unknown) number of clusters and $n_u$ is the number of users. Let $n^j$ be the (unknown) number of users in cluster $j$ and $p^j$ be the (unknown) frequency for cluster $j$, which is the sum of (unknown) user frequencies in cluster $j$. The following main theorem bounds the cumulative regret of our algorithm \sclub.
\begin{theorem}
\label{thm:main}
Suppose the clustering structure over the users, user frequencies, and items satisfy the assumptions stated in \cref{sec:setting} with gap parameters $\gamma, \gamma_p > 0$ and item regularity parameter $0 < \lambda_x \le 1$. Let $\alpha_{\theta} = 4 R \sqrt{d/\lambda_x}$, $\alpha_p = 2$ and
$\beta = R \sqrt{d \ln(1 + T / d) + 2 \ln(4 m n_u)}$. Then the cumulative regret of the algorithm \sclub after $T$ rounds satisfies
\begin{align} \label{eq:SCLUBregret}
R(T) &\le \sum_{j=1}^m 4 \beta \sqrt{d p^j T \ln(T/d)}  \notag \\
&\quad + O\left( \left(\frac{1}{\gamma_p^2} + \frac{n_u}{\gamma^2 \lambda_x^3}\right) \ln(T) \right)\\
&= O(d\sqrt{m T} \ln(T))\,.
\end{align}
\end{theorem}
\begin{proof}[sketch]
The proof is mainly based on two parts. The first part bounds the exploration rounds to guarantee the clusters partitioned correctly. The second part is to estimate regret bounds for linear bandits after the clusters are partitioned correctly.

By Chernoff-Hoeffding inequality, we prove when
$
\tau \ge O\left(\frac{1}{\gamma_p^2}\ln(T) + \frac{1}{p_i}\ln(T) \right)
$
the distance of $p_i$ and $\hat{p}_i$ is less than $\gamma_p/2$. Then the users with different frequencies will be assigned to different clusters. Thus, the cluster containing user $i$ only contains users with the same frequency probabilities, for all user $i$.

Then by the assumption of item regularity, when
$
\tau \ge O \left(\frac{1}{\gamma_p^2}\ln(T) + \frac{1}{p_i}\ln(T) + \frac{1}{p_i \gamma^2 \lambda_x^3} \ln(T) \right)
$
the estimated weight vectors will be accurate enough for all the users with the frequency probability same as $p_i$. The 2-norm radius for the confidence ellipsoid of the weight vector will be less than $\gamma/2$. Thus the split and merge operations will function correctly for users. Next, we prove that if two clusters containing the users of same frequency probabilities have accurate estimates for weight vectors and frequencies, the combined cluster inherits the accuracies with high probability. Thus, the combination steps can continue until it contains all users of an underlying cluster.

After the cluster $j$ with frequency probability $p^j$ is correctly partitioned, the recommendation is based on the estimates of cluster weight vector. The regret for the second part reduces to linear bandits with $p^j T$ rounds.
\end{proof}

The proof is much different from that of \club \cite{gentile2014online} with the additional frequency probability and merge operation. 
The full proof is put in 
\ifsup
\cref{sec:proof of sclub}.
\else
the supplementary materials.
\fi

\subsection{Discussions}

First, we compare our regret bound with that of \club. Since \club is designed and proved only for uniform distribution over users, we first generalize the regret bound of \club to the setting of arbitrary distribution over users.
\begin{theorem}
\label{thm:club}
Under the same setting with \sclub, the cumulative regret of the algorithm, \club, after $T$ rounds satisfies
\begin{align*} \label{eq:CLUBregret}
R(T) &\le \sum_{j=1}^m 4 \beta \sqrt{d p^j T \ln(T/d)} + O\left(\frac{1}{p_{\min} \gamma^2 \lambda_x^3} \ln(T) \right)\\
  &= O(d \sqrt{m T} \ln(T))\,.
\end{align*}
\end{theorem}
Note that in the degenerate setting where the distribution $p$ over users is uniform, $p_{\min} = 1 / n_u$ and the regret bound recovers the one in the previous work \cite{gentile2014online}.

In real applications, there might be a lot of infrequent users with frequency close to $0$. Since the algorithm \club initializes with a complete graph or a random connected graph, there would be a lot of infrequent users connecting with frequent users. Infrequent users receive a recommendation once a long time. Thus it takes a long time to differentiate them from frequent users, or to be split into a different cluster (connected component) from frequent users. Then the recommendations for frequent users would be polluted by these infrequent and dissimilar users for a long time. The length of rounds to separate infrequent dissimilar users is proportional to $1/p_{\min}$, the minimal frequency over all users. $p_{\min}$ could be arbitrarily small in real applications, the $\log$ term in regret bound for \club is not satisfactory. Due to an additional care on frequencies, \sclub could get rid of the uncontrollable term $1 / p_{\min}$.

The assumption that users with same weight vectors have same frequencies is based on the intuition that frequent users and infrequent users usually have different preferences. It can be relaxed to that users with same weight vectors have close frequencies. For more general cases, a new structure of nested clusters could help where the infrequent users could use the information of frequent users but not vice versa. We leave this as interesting future work.

Second, our result can be generalized by allowing multiple users each round. In applications the recommender system will update itself every a fixed time slot, and during each time slot, multiple users will be served using the same history information. 
\begin{theorem}
\label{thm:multiple users}
Under the same setting with \sclub, each user appear with frequency $p_i \in [0, 1]$, then the cumulative regret of the algorithm after $T$ rounds satisfies
\begin{align*}
R(T) &\le \sum_{j=1}^m 4 \beta \sqrt{d p^j T \ln(T/d)} \\
&\quad + O\left( \left(\frac{\bar{p}}{\gamma_p^2} + \frac{n_u}{\gamma^2 \lambda_x^3}\right) \ln(T) \right)
\end{align*}
where $\bar{p} = \sum_{i = 1}^{n_u} p_i$ is the sum of all user frequencies.
\end{theorem}
Note this recovers the regret bound in \cref{thm:main} when $\bar{p} = 1$. The regret bound for \club can also be generalized in a similar way.

Third, if the clustering structure is known, the setting would be equivalent to $m$ independent linear bandits, each with the expected number of rounds $p^j T$. Then the regret bound would be $O(d\sum_{j=1}^m\sqrt{p^j T} \ln(T))$ by \cite{abbasi2011improved} which matches the main term of ours \eqref{eq:SCLUBregret}. The upper bound reaches its maximum $O(d\sqrt{m T} \ln(T))$ when $p^1 = \cdots = p^m = 1/m$. Also the regret lower bound in this case is $\Omega(\sum_{j=1}^m \sqrt{d p^j T})$ by \cite{dani2008stochastic} which matches the main term of \eqref{eq:SCLUBregret} up to a term of $\sqrt{d} \ln(T)$.

Fourth, we would like to state the additional benefit of \sclub in privacy protection. To protect user privacy, when a user $i$ comes to be served, the learning agent will have the only access to the information of user $i$ and some aggregate information among users, but does not have the access to other individual users' information. In this sense, lack of other individual users' information will make existing graph-based methods \cite{gentile2014online,li2016collaborative} inapplicable. Our algorithm \sclub only uses information of user $i$ and aggregate information of clusters to split and merge.

Due to space limit, more discussions are put in
\ifsup
\cref{sec:more discussions}.
\else
the supplementary materials.
\fi

\section{Experiments}

\begin{figure*}[thb!]
\centering
{\small (1) Synthetic Dataset}\\
\includegraphics[width = 0.24\textwidth]{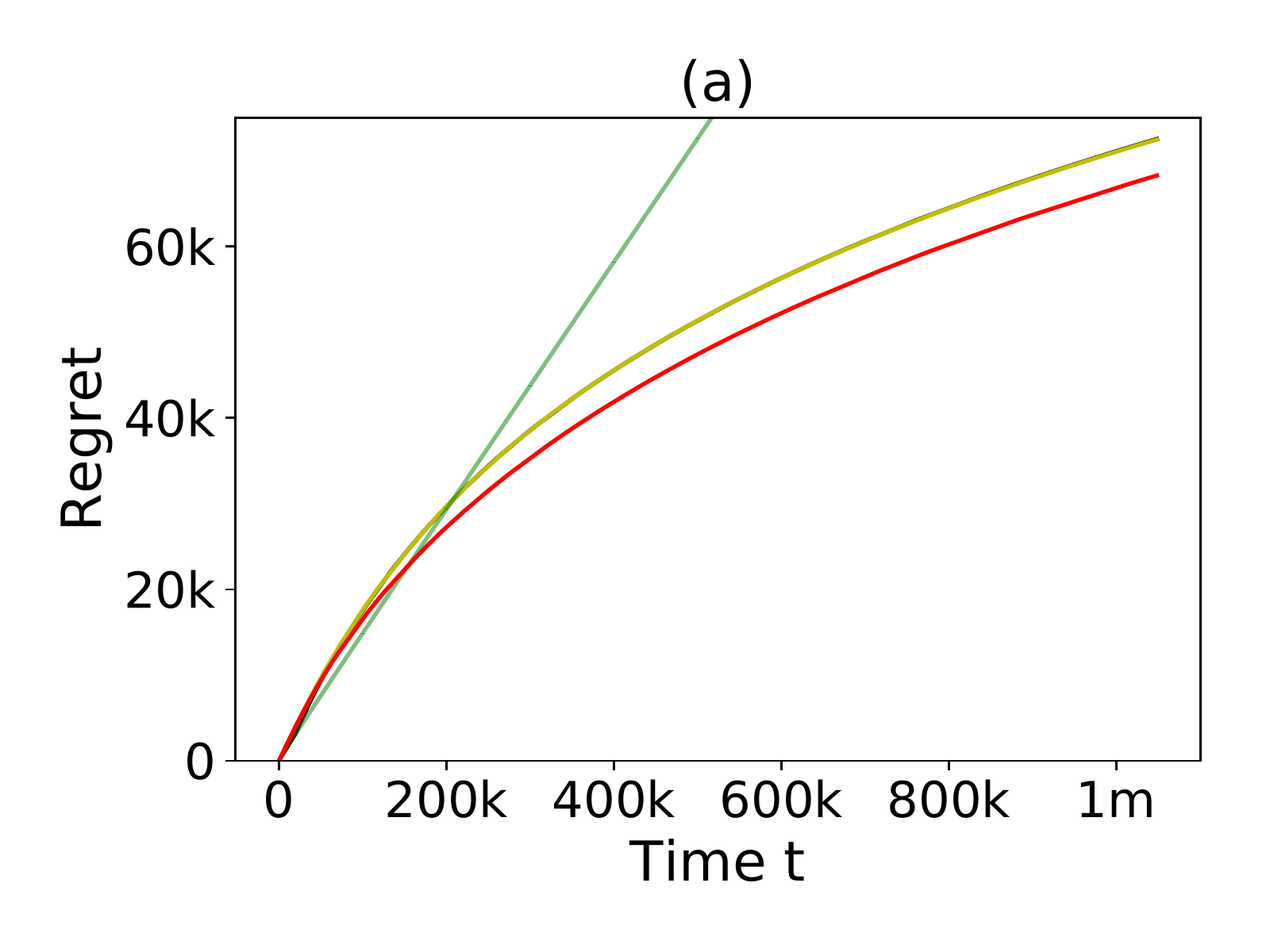}
\includegraphics[width = 0.24\textwidth]{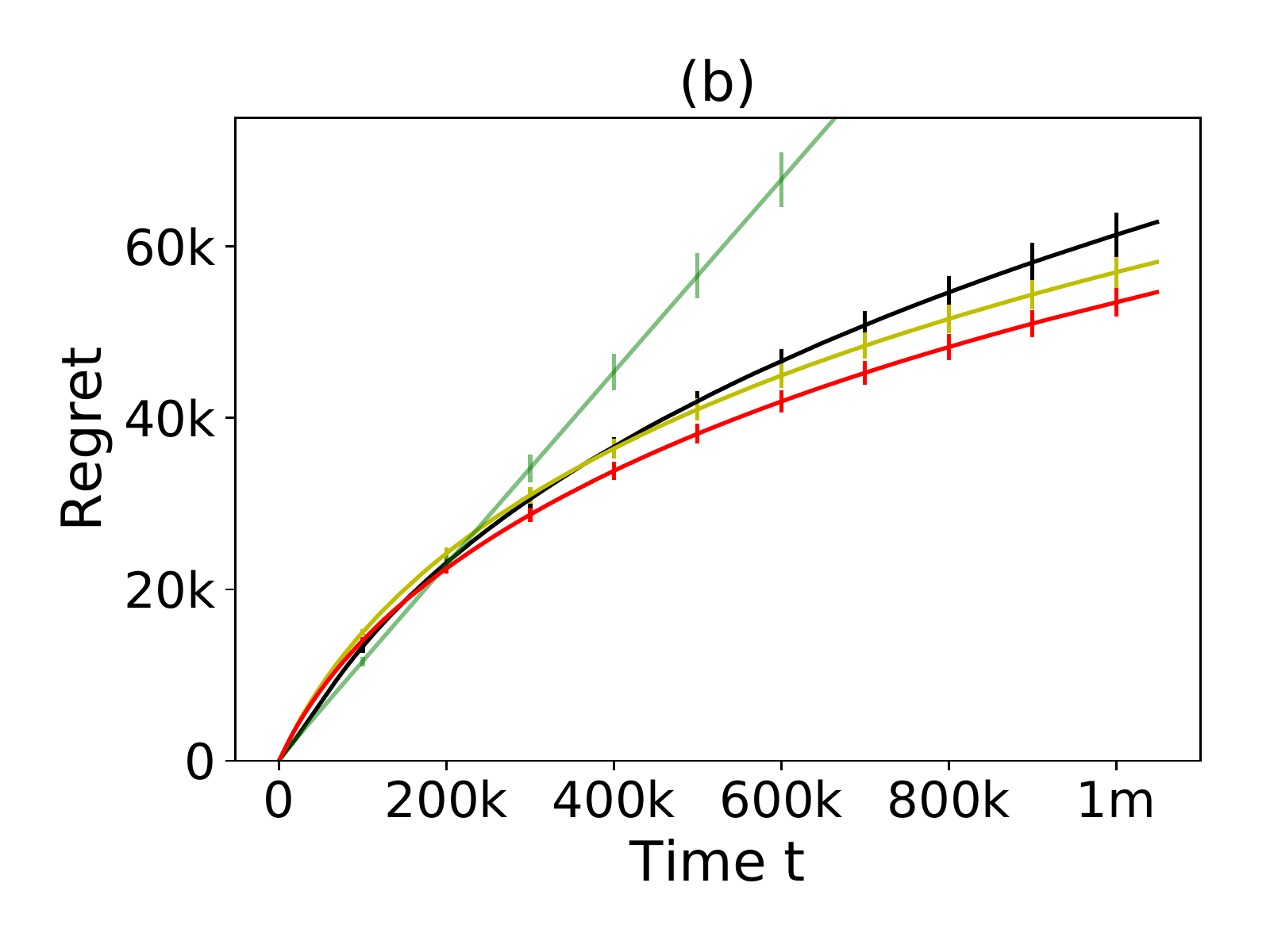}
\includegraphics[width = 0.24\textwidth]{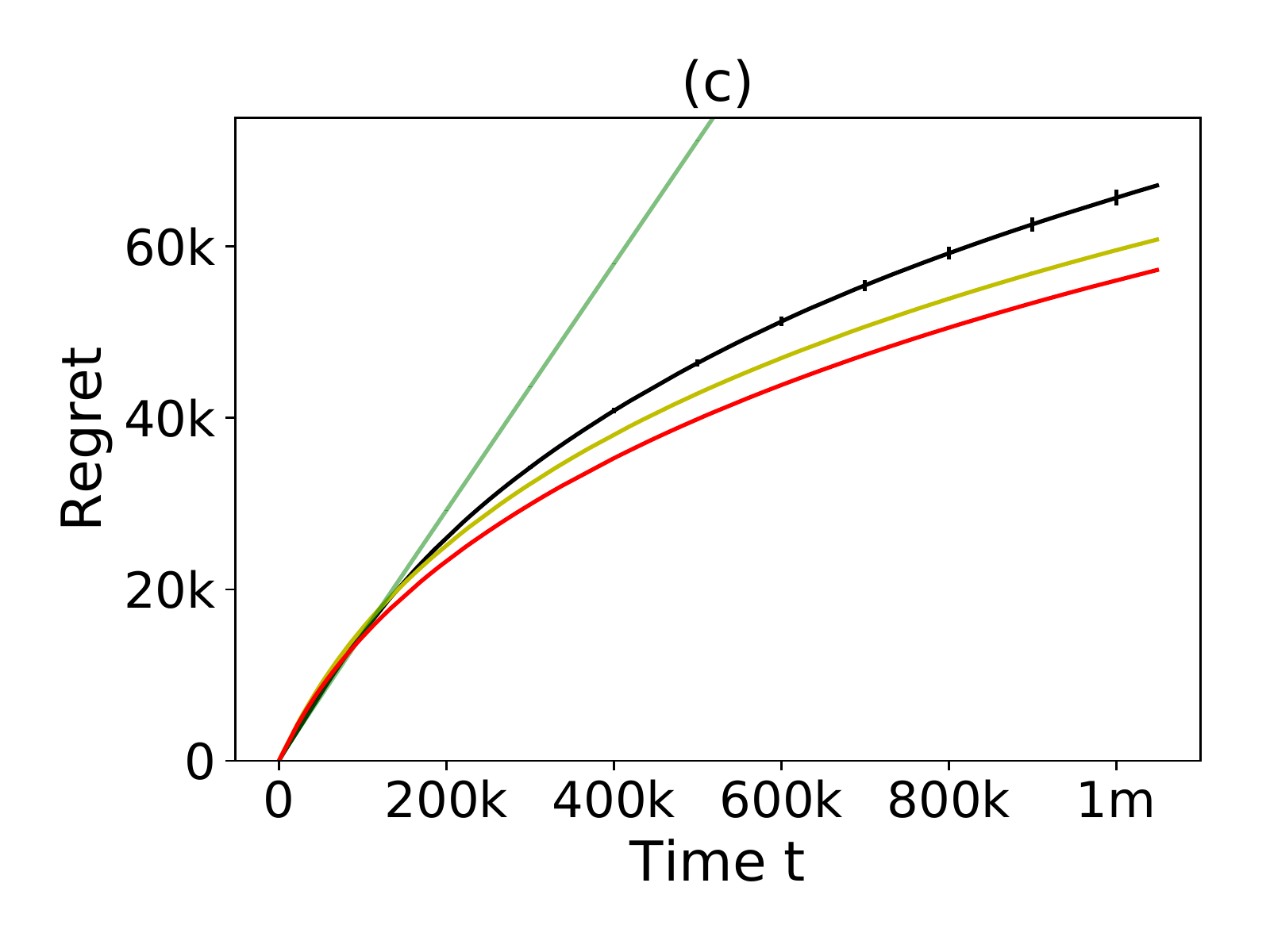}\\
{\small (2) MovieLens and Yelp Datasets}\\
\includegraphics[width = 0.24\textwidth]{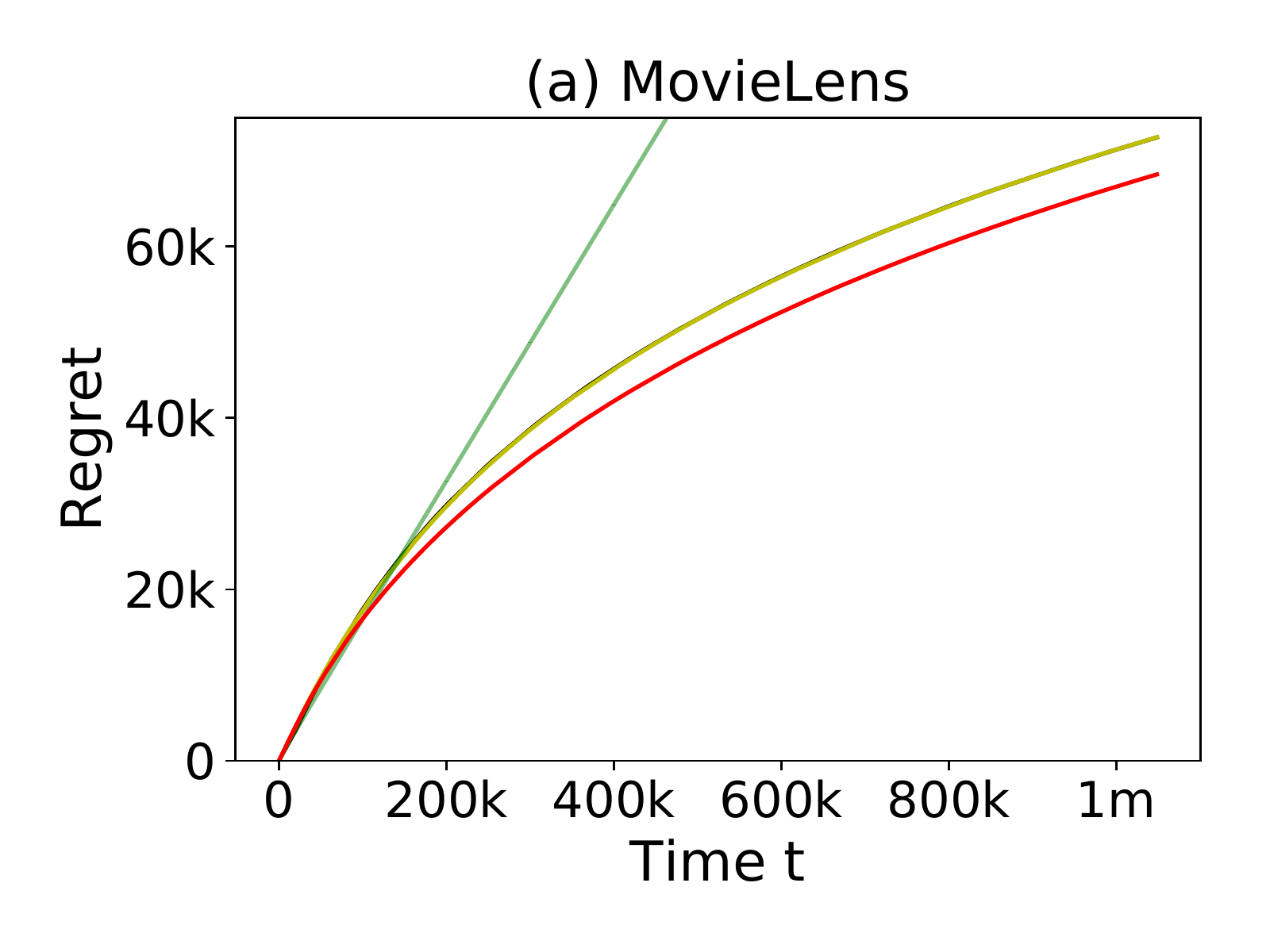}
\includegraphics[width = 0.24\textwidth]{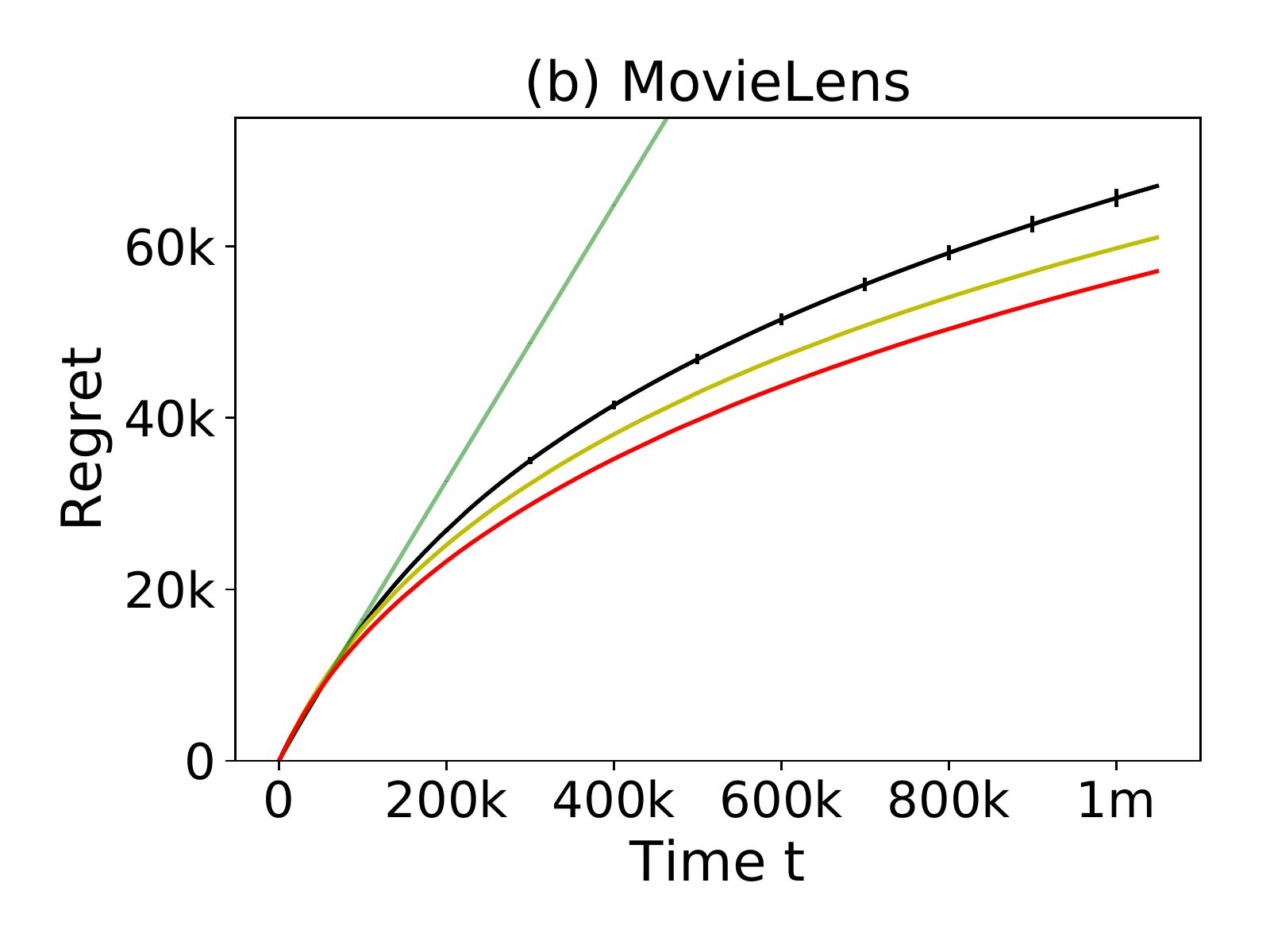}
\includegraphics[width = 0.24\textwidth]{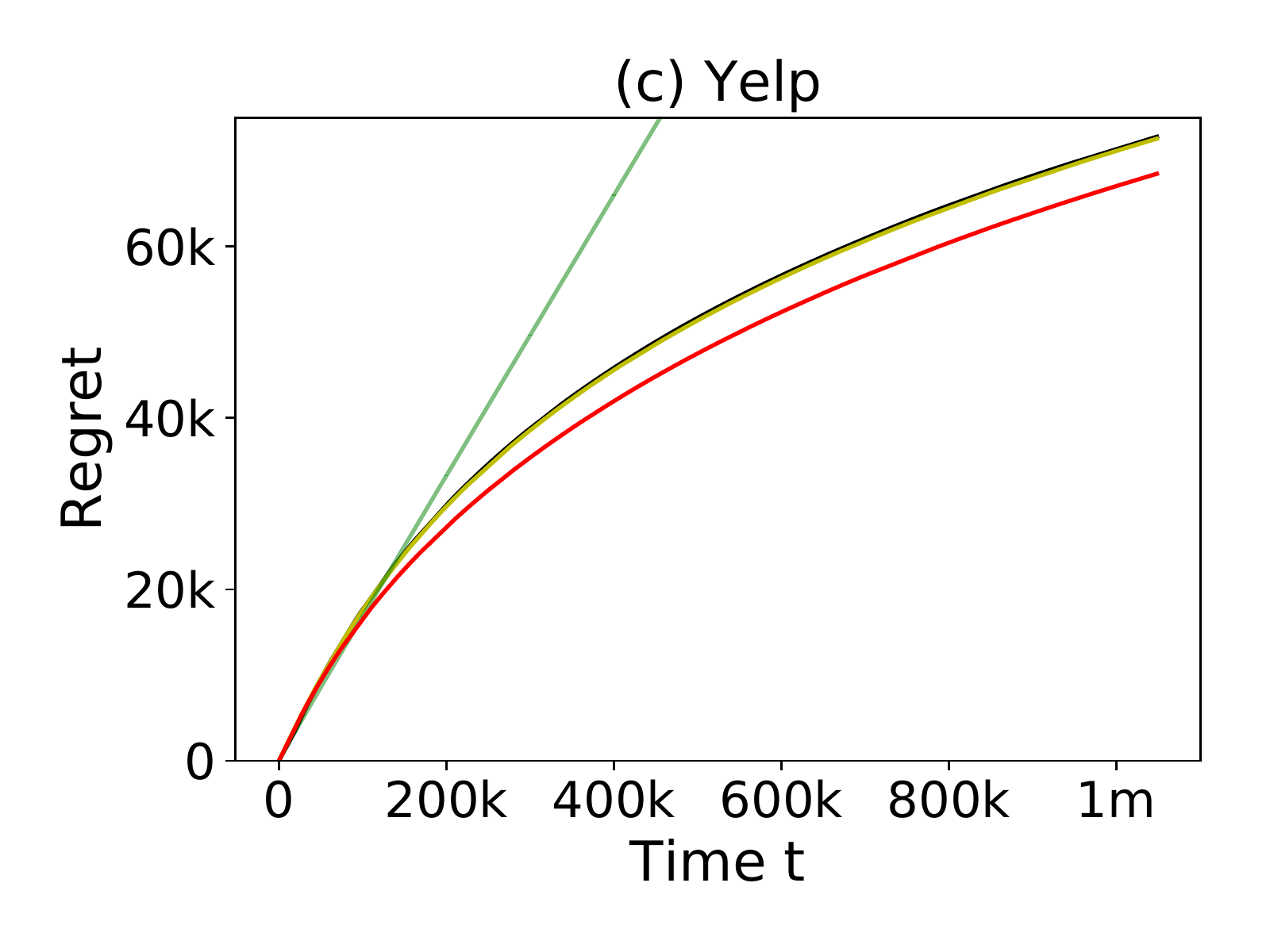}
\includegraphics[width = 0.24\textwidth]{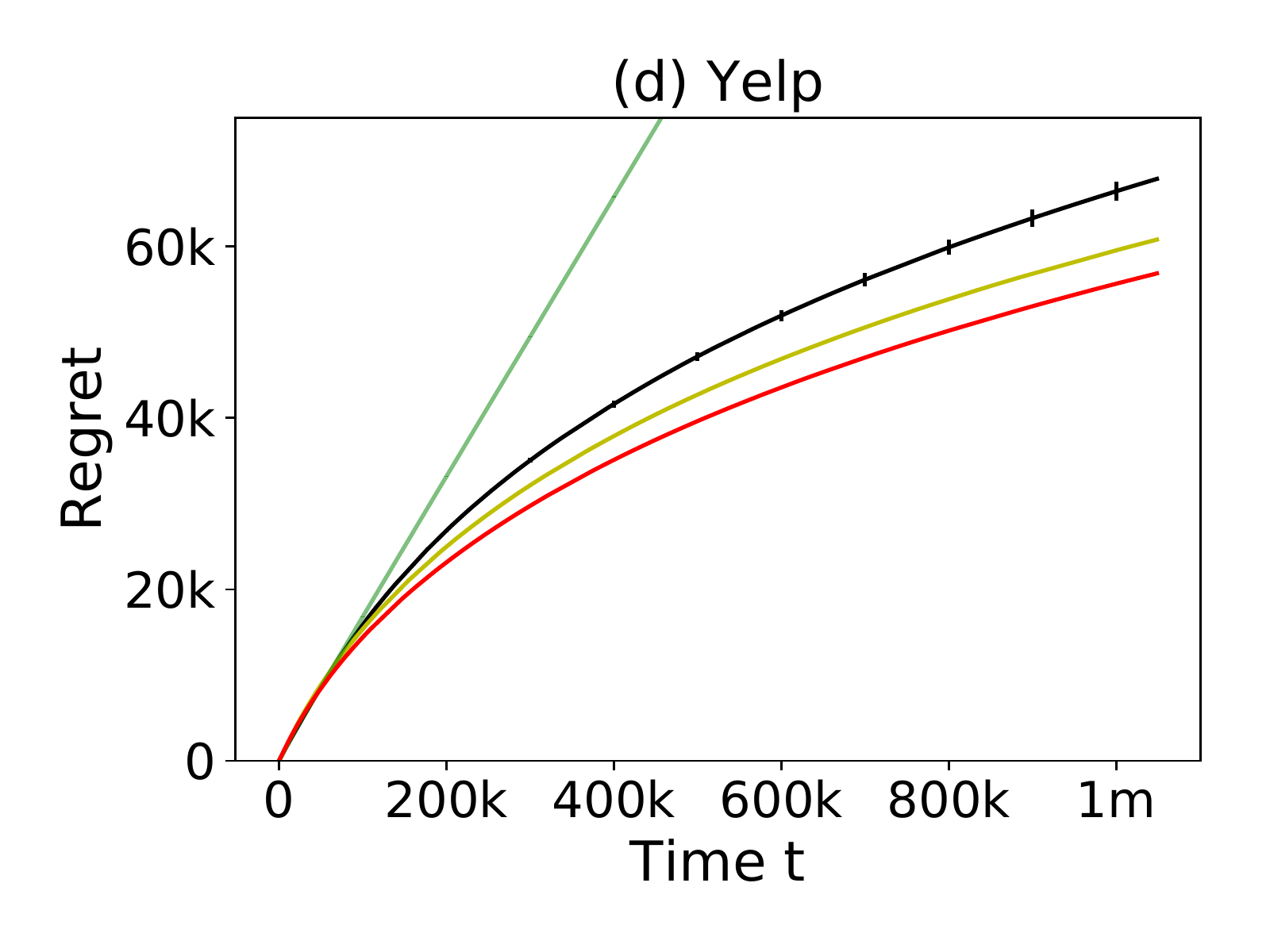}
\caption{
The figures compare \sclub (red) with \club (black), \linucbone (green) and \linucbind (yellow). The first row is for synthetic experiments and the second row is for real datasets, MovieLens and Yelp.
All the experiments are of $n_u=10^3$ users with $d=20, L=20$. We set $m=10$ for (1) synthetic experiments. (1a)(2a)(2c) are of uniform distribution over users, (1b) is of arbitrary distribution over clusters, where the users in the same cluster have the same frequency probabilities, (1c)(2b)(2d) are of arbitrary distribution over users. All results are averaged under $10$ random runs and the errorbars are computed by standard errors, which are standard deviations divided by $\sqrt{10}$.
\vspace{-0.3cm}
}
\label{fig:results}
\end{figure*}

\sclub algorithm is compared with \club \cite{gentile2014online}, \linucbone which uses a single estimated weight vector for all users and \linucbind which uses a separate estimated weight vector for each user on both synthetic and real datasets.

\subsection{Synthetic experiments} 

We consider a setting of $n_u=10^3$ users with $m=10$ clusters where each cluster contains equal number of users. The weight vectors $\theta_1, \ldots, \theta_m$ and item vectors at each round are first randomly drawn in $d-1$ ($d=20$) dimension with each entry a standard Gaussian variable, then normalized, added one more dimension with constant $1$, and divided by $\sqrt{2}$. The transformation is as follows
\begin{equation}
\label{eq:transform x}
x\mapsto \left(\frac{x}{\sqrt{2}\norm{x}},\ \frac{1}{\sqrt{2}}\right)\,. 
\end{equation}
This transformation on both the item vector $x$ and weight vector $\theta$ is to guarantee the mean $\ip{\theta, x}$ lies in $[0,1]$.
The parameters in all algorithms take theoretical values.
The comparisons on theoretical computational complexities of these algorithms is put in
\ifsup
\cref{sec:complexity of implementation}.
\else
supplementary materials.
\fi
The evolution of the regret as a function of time is shown in the first row of \cref{fig:results}
\footnote{The parameters of \club are theoretical values and might be suboptimal. We do not optimize the parameters of any algorithm and just use theoretical values of the parameters in all algorithms including \sclub.}
. The regrets at the end and total running times are given in
\ifsup
\cref{table:regrets} and \cref{table:running time}.
\else
supplementary materials.
\fi
The left figure (a) is under the setting of uniform distribution over all users where each user has a frequency $1/n_u$; middle figure (b) is under the setting of arbitrary distribution over clusters where the users in the same cluster have the same frequency probability; right figure (c) is under the setting of arbitrary distribution over users.

The performance of \sclub improves over \club by $5.94\%$ (\cref{fig:results}(1a)) even in the setting of \club where the frequency distribution over users is uniform. The reason is that \club uses connected components of graphs to represent clusters and if it wants to identify the underlying cluster of a user $i$, it has to delete every bad edge of its neighbors (and neighbors of neighbors, etc.). In comparison, \sclub is much faster, will split a user $i$ out if it finds any inconsistency from her current running cluster, and merge her to any consistent running cluster. The improvement is enlarged to be $13.02\%$ (\cref{fig:results}(1b)) in the setting of \cref{sec:setting} where users in the same cluster have the same frequency probability but users in different clusters might have different frequencies. \sclub is robust in the setting of arbitrary distribution over users where the assumptions in \cref{sec:setting} might fail and can still improve over \club by $14.69\%$ (\cref{fig:results}(1c)).

\subsection{Real datasets} 

We use the $20m$ MovieLens dataset \cite{harper2016movielens} which contains $20$ million ratings for $2.7\times 10^4$ movies by $1.38 \times 10^5$ users and Yelp dataset\footnote{\url{http://www.yelp.com/dataset\_challenge}} 
which contains $4.7$ million ratings of $1.57 \times 10^5$ restaurants from $1.18$ million users. The Yelp dataset is more sparse than MovieLens. For each of the two real datasets, we extract $10^3$ items with most ratings and $n_u = 10^3$ users who rate most. Then use the rating matrix to derive feature vectors of $d-1$ ($d=20$) dimension for all users by singular-value decomposition (SVD). The feature vectors are also processed as \eqref{eq:transform x} and the resulting vectors are regarded as the underlying weight vectors of each user. 

The performances are shown in \cref{fig:results}(2) where (2a)(2b) are for MovieLens dataset and (2c)(2d) are for Yelp dataset. (2a)(2c) are under the setting of uniform distribution over users and (2b)(2d) are under the setting of arbitrary distribution over users. Since there is no underlying clustering over the users, we do not experiment on the setting of arbitrary distribution over clusters like \cref{fig:results}(1b). The algorithms of \sclub and \club adaptively find finer and finer clustering as more data flows in. The performance of \sclub improves over \club by $5.94\%, 14.84\%, 5.94\%, 16.24\%$ respectively.

\section{Conclusions and Future Work}
In this paper, we extend the existing setting for online clustering of bandits to include non-uniform distribution over users. The new set-based algorithm together with both split and merge operations clusters users adaptively and is proven with a better theoretical guarantee on cumulative regrets than the existing works. Our algorithm also has an additional benefit of privacy protection. The experiments on both synthetic and real datasets show that our algorithm is consistently better then the existing works.

One interesting future direction is to explore the asymmetric relationships between users,
whereas all existing works use symmetric relationships including ours. For example, recommendations of low-frequency users can use information (or feedback) from high-frequency users, but not vice versa. This idea can be extended by using nested clusters
 where an underlying cluster of frequent users can be a subset of another underlying cluster containing many infrequent users. 
Another interesting problem is to generalize the idea of our method to collaborative filtering of both users and items such as \citet{li2016collaborative}. Also extending \cite{li2016collaborative} to the general setting of changing item set would be a challenge one.

\section*{Acknowledgements}
This work was partially supported by The Chinese University of Hong Kong Direct Grants [ID: 4055073 and 4055104]. Wei Chen is partially supported by the National Natural Science Foundation of China (Grant No. 61433014).

\bibliographystyle{named}
\bibliography{ref}

\begin{thebibliography}{}

\bibitem[\protect\citeauthoryear{Abbasi-Yadkori \bgroup \em et al.\egroup
  }{2011}]{abbasi2011improved}
Yasin Abbasi-Yadkori, D{\'a}vid P{\'a}l, and Csaba Szepesv{\'a}ri.
\newblock Improved algorithms for linear stochastic bandits.
\newblock In {\em Advances in Neural Information Processing Systems}, pages
  2312--2320, 2011.

\bibitem[\protect\citeauthoryear{Abe and Long}{1999}]{abe1999associative}
Naoki Abe and Philip~M Long.
\newblock Associative reinforcement learning using linear probabilistic
  concepts.
\newblock In {\em Proceedings of the Sixteenth International Conference on
  Machine Learning}, pages 3--11. Morgan Kaufmann Publishers Inc., 1999.

\bibitem[\protect\citeauthoryear{Aggarwal}{2016}]{aggarwal2016recommender}
Charu Aggarwal.
\newblock {\em Recommender Systems}.
\newblock Springer, 2016.

\bibitem[\protect\citeauthoryear{Bubeck \bgroup \em et al.\egroup
  }{2012}]{bubeck2012regret}
S{\'e}bastien Bubeck, Nicolo Cesa-Bianchi, et~al.
\newblock Regret analysis of stochastic and nonstochastic multi-armed bandit
  problems.
\newblock {\em Foundations and Trends{\textregistered} in Machine Learning},
  5(1):1--122, 2012.

\bibitem[\protect\citeauthoryear{Christakopoulou and
  Banerjee}{2018}]{christakopoulou2018learning}
Konstantina Christakopoulou and Arindam Banerjee.
\newblock Learning to interact with users: A collaborative-bandit approach.
\newblock In {\em Proceedings of the 2018 SIAM International Conference on Data
  Mining}, pages 612--620. SIAM, 2018.

\bibitem[\protect\citeauthoryear{Chu \bgroup \em et al.\egroup
  }{2011}]{chu2011contextual}
Wei Chu, Lihong Li, Lev Reyzin, and Robert~E Schapire.
\newblock Contextual bandits with linear payoff functions.
\newblock In {\em AISTATS}, volume~15, pages 208--214, 2011.

\bibitem[\protect\citeauthoryear{Dani \bgroup \em et al.\egroup
  }{2008}]{dani2008stochastic}
Varsha Dani, Thomas~P Hayes, and Sham~M Kakade.
\newblock Stochastic linear optimization under bandit feedback.
\newblock In {\em COLT}, pages 355--366, 2008.

\bibitem[\protect\citeauthoryear{Gentile \bgroup \em et al.\egroup
  }{2014}]{gentile2014online}
Claudio Gentile, Shuai Li, and Giovanni Zappella.
\newblock Online clustering of bandits.
\newblock In {\em Proceedings of the 31st International Conference on Machine
  Learning}, pages 757--765, 2014.

\bibitem[\protect\citeauthoryear{Gentile \bgroup \em et al.\egroup
  }{2017}]{gentile2017context}
Claudio Gentile, Shuai Li, Purushottam Kar, Alexandros Karatzoglou, Giovanni
  Zappella, and Evans Etrue.
\newblock On context-dependent clustering of bandits.
\newblock In {\em International Conference on Machine Learning}, pages
  1253--1262, 2017.

\bibitem[\protect\citeauthoryear{Harper and
  Konstan}{2016}]{harper2016movielens}
F~Maxwell Harper and Joseph~A Konstan.
\newblock The movielens datasets: History and context.
\newblock {\em ACM Transactions on Interactive Intelligent Systems (TiiS)},
  5(4):19, 2016.

\bibitem[\protect\citeauthoryear{Hoeffding}{1963}]{hoeffding1963probability}
Wassily Hoeffding.
\newblock Probability inequalities for sums of bounded random variables.
\newblock {\em Journal of the American statistical association},
  58(301):13--30, 1963.

\bibitem[\protect\citeauthoryear{Katariya \bgroup \em et al.\egroup
  }{2017a}]{katariya2017bernoulli}
Sumeet Katariya, Branislav Kveton, Csaba Szepesv{\'a}ri, Claire Vernade, and
  Zheng Wen.
\newblock Bernoulli rank-1 bandits for click feedback.
\newblock In {\em Proceedings of the 26th International Joint Conference on
  Artificial Intelligence}, pages 2001--2007. AAAI Press, 2017.

\bibitem[\protect\citeauthoryear{Katariya \bgroup \em et al.\egroup
  }{2017b}]{katariya2017stochastic}
Sumeet Katariya, Branislav Kveton, Csaba Szepesvari, Claire Vernade, and Zheng
  Wen.
\newblock Stochastic rank-1 bandits.
\newblock In {\em Artificial Intelligence and Statistics}, pages 392--401,
  2017.

\bibitem[\protect\citeauthoryear{Korda \bgroup \em et al.\egroup
  }{2016}]{korda2016distributed}
Nathan Korda, Balazs Szorenyi, and Shuai Li.
\newblock Distributed clustering of linear bandits in peer to peer networks.
\newblock In {\em The 33rd International Conference on Machine Learning
  (ICML)}, 2016.

\bibitem[\protect\citeauthoryear{Kveton \bgroup \em et al.\egroup
  }{2017}]{kveton2017stochastic}
Branislav Kveton, Csaba Szepesvari, Anup Rao, Zheng Wen, Yasin Abbasi-Yadkori,
  and S~Muthukrishnan.
\newblock Stochastic low-rank bandits.
\newblock {\em arXiv preprint arXiv:1712.04644}, 2017.

\bibitem[\protect\citeauthoryear{Kwon}{2018}]{kwon2018overlapping}
Jin~Kyoung Kwon.
\newblock Overlapping clustering of contextual bandits with nmf techniques.
\newblock 2018.

\bibitem[\protect\citeauthoryear{Lattimore and Szepesv\'{a}ri}{2018}]{LS18book}
Tor Lattimore and Csaba Szepesv\'{a}ri.
\newblock {\em Bandit Algorithms}.
\newblock preprint, 2018.

\bibitem[\protect\citeauthoryear{Li and Zhang}{2018}]{li2018online}
Shuai Li and Shengyu Zhang.
\newblock Online clustering of contextual cascading bandits.
\newblock In {\em Thirty-Second AAAI Conference on Artificial Intelligence},
  2018.

\bibitem[\protect\citeauthoryear{Li \bgroup \em et al.\egroup
  }{2010}]{li2010contextual}
Lihong Li, Wei Chu, John Langford, and Robert~E Schapire.
\newblock A contextual-bandit approach to personalized news article
  recommendation.
\newblock In {\em Proceedings of the 19th international conference on World
  wide web}, pages 661--670. ACM, 2010.

\bibitem[\protect\citeauthoryear{Li \bgroup \em et al.\egroup
  }{2016a}]{li2016collaborative}
Shuai Li, Alexandros Karatzoglou, and Claudio Gentile.
\newblock Collaborative filtering bandits.
\newblock In {\em Proceedings of the 39th International ACM SIGIR conference on
  Research and Development in Information Retrieval}, pages 539--548. ACM,
  2016.

\bibitem[\protect\citeauthoryear{Li \bgroup \em et al.\egroup
  }{2016b}]{li2016contextual}
Shuai Li, Baoxiang Wang, Shengyu Zhang, and Wei Chen.
\newblock Contextual combinatorial cascading bandits.
\newblock In {\em Proceedings of The 33rd International Conference on Machine
  Learning}, pages 1245--1253, 2016.

\bibitem[\protect\citeauthoryear{Li}{2016}]{li2016art}
Shuai Li.
\newblock {\em The art of clustering bandits.}
\newblock PhD thesis, Universit{\`a} degli Studi dell'Insubria, 2016.

\bibitem[\protect\citeauthoryear{Li}{2019}]{li2019online}
Shuai Li.
\newblock Online clustering of contextual cascading bandits.
\newblock {\em arXiv preprint arXiv:1711.08594v2}, 2019.

\bibitem[\protect\citeauthoryear{Linden \bgroup \em et al.\egroup
  }{2003}]{linden2003amazon}
Greg Linden, Brent Smith, and Jeremy York.
\newblock Amazon. com recommendations: Item-to-item collaborative filtering.
\newblock {\em IEEE Internet computing}, 7(1):76--80, 2003.

\bibitem[\protect\citeauthoryear{MacQueen}{1967}]{macqueen1967some}
James MacQueen.
\newblock Some methods for classification and analysis of multivariate
  observations.
\newblock In {\em Proceedings of the fifth Berkeley symposium on mathematical
  statistics and probability}, volume~1, pages 281--297. Oakland, CA, USA,
  1967.

\bibitem[\protect\citeauthoryear{Nguyen and Lauw}{2014}]{nguyen2014dynamic}
Trong~T Nguyen and Hady~W Lauw.
\newblock Dynamic clustering of contextual multi-armed bandits.
\newblock In {\em Proceedings of the 23rd ACM International Conference on
  Conference on Information and Knowledge Management}, pages 1959--1962. ACM,
  2014.

\bibitem[\protect\citeauthoryear{Rusmevichientong and
  Tsitsiklis}{2010}]{rusmevichientong2010linearly}
Paat Rusmevichientong and John~N Tsitsiklis.
\newblock Linearly parameterized bandits.
\newblock {\em Mathematics of Operations Research}, 35(2):395--411, 2010.

\bibitem[\protect\citeauthoryear{Turgay \bgroup \em et al.\egroup
  }{2018}]{turgay2018multi}
Eralp Turgay, Doruk Oner, and Cem Tekin.
\newblock Multi-objective contextual bandit problem with similarity
  information.
\newblock In {\em International Conference on Artificial Intelligence and
  Statistics}, pages 1673--1681, 2018.

\bibitem[\protect\citeauthoryear{Woo \bgroup \em et al.\egroup
  }{2014}]{woo2014cluster}
Choong-Wan Woo, Anjali Krishnan, and Tor~D Wager.
\newblock Cluster-extent based thresholding in fmri analyses: pitfalls and
  recommendations.
\newblock {\em Neuroimage}, 91:412--419, 2014.

\end{thebibliography}

\ifsup
\appendix
\newpage
\onecolumn

\section{Proof of \cref{thm:main}}
\label{sec:proof of sclub}

Recall that $\tau$ denotes the index of time step. First we prove a sufficent condition on the number of collected samples when will the estimated weight vectors for individual users are good enough.
\begin{lemma}
\label{lem:theta-hat-good-when-T}
Fix any user $i$. Let $A(\delta) = \frac{1024}{\lambda_x^2} \ln \frac{512d}{\lambda_x^2 \delta}$. It holds with probability at least $1 - \delta$ that
\begin{align*}
\norm{\hat{\theta}_{i,\tau} - \theta_i} \le \frac{R \sqrt{d\ln\left(1 + \frac{T_{i, \tau}}{d}\right) + 2\ln\frac{2}{\delta} } + \sqrt{A(\delta/2)}}{\sqrt{T_{i, \tau} \lambda_x / 8}} < \frac{\gamma}{4}
\end{align*}
when $T_{i, \tau} > \frac{8192}{\gamma^2 \lambda_x^3} \ln \frac{1024 d}{\lambda_x^2 \delta} =: B(\delta)$.
\end{lemma}
\begin{proof}
By \cite[Claim 1]{gentile2014online} and \cite[Lemma 7]{li2019online}, with probability at least $1 - \delta_1$,
\begin{align*}
\lambda_{\min}(S_{i, \tau}) \ge T_{i, \tau} \lambda_x/8
\end{align*}
when $T_{i, \tau} \ge A(\delta_1)$. Further by \cite[Theorem 2]{abbasi2011improved}, with probability at least $1 - \delta_1 - \delta_2$,
\begin{align*}
&\norm{\hat{\theta}_{i, \tau} - \theta_{i}}_{S_{i, \tau - 1}} \le R \sqrt{d\ln\left(1 + \frac{T_{i, \tau}}{d} \right) + 2\ln\frac{1}{\delta_2}} + \sqrt{A(\delta_1)}
\end{align*}
when $T_{i, \tau} \ge A(\delta_1)$, where we replace $\lambda$ in that theorem by $A(\delta_1) \lambda_x / 8$. Thus with probability at least $1 - \delta_1 - \delta_2$, 
\begin{align*}
&\norm{\hat{\theta}_{i,\tau} - \theta_i} \le \frac{\norm{\hat{\theta}_{i, \tau} - \theta_{i}}_{S_{i, \tau - 1}}}{\sqrt{\lambda_{\min}(S_{i, \tau})}} \le \frac{R \sqrt{d\ln\left(1 + \frac{T_{i, \tau}}{d} \right) + 2\ln\frac{1}{\delta_2}} + \sqrt{A(\delta_1)}}{\sqrt{T_{i, \tau} \lambda_x/8}}
\end{align*}
when $T_{i, \tau} \ge A(\delta_1)$. The last term is less than $\gamma/4$ when
\begin{align*}
T_{i, \tau} > \frac{8}{\gamma^2 \lambda_x}\max\left\{A(\delta_1), \quad 4R^2d\log\frac{32R^2}{\gamma^2 \lambda_x}, \quad 4R^2\ln\frac{1}{\delta_2} \right\}\,.
\end{align*}
by Lemma \ref{lem:T-for-ratio-less-than-gamma}.

Putting the conditions on $T_{i, \tau}$ together and letting $\delta_1 = \delta_2 = \delta / 2$, we need $T > \frac{8192}{\gamma^2 \lambda_x^3} \ln \frac{1024 d}{\lambda_x^2 \delta}$
under the assumption on the numeric relations between parameters: $\ln\frac{2}{\delta} \ge d \ln\frac{32 R^2}{\gamma^2 \lambda_x}$ and $R^2 < \frac{2048}{\gamma^2 \lambda_x^3}$. Here we assume that $\delta$ can be as small as $\frac{1}{\sqrt{T}}$ and $T$ is large enough.
\end{proof}


	
\begin{lemma}
\label{lem:p-hat-good-enough-when-t}
Fix any user $i$. With probability at least $1 - \delta$, it holds that
\begin{equation}
\abs{\hat{p}_{i, \tau} - p_i} \le \sqrt{\frac{\ln(2 c \tau^2 / \delta)}{2\tau}} < \frac{\gamma_p}{4}
\end{equation}
when $\tau \ge \frac{16}{\gamma_p^2}\ln\frac{512 c}{\gamma_p^4 \delta}$.
\end{lemma}
\begin{proof}
Note that $T_{i, \tau}$ is a sum of $\tau$ independent Bernoulli random variable with mean $p_i$. The first inequality is obtained by Chernoff-Hoeffding inequality (Lemma \ref{lem:additive chernoff bound}) with probability at least $1 - \delta$. The second inequality holds when $\tau \ge \frac{16}{\gamma_p^2}\ln\frac{512 c}{\gamma_p^4 \delta}$ by \cite[Lemma 9]{li2019online}.
\end{proof}

With the help of the above lemmas, the main theorem is ready to be proved.

\begin{proof}[of Theorem \ref{thm:main}]
By Lemma \ref{lem:p-hat-good-enough-when-t}, it holds with probability at least $1 - \delta_1$ that
\begin{align*}
\abs{\hat{p}_{i, \tau} - \hat{p}_{i', \tau}} > \gamma_p / 2 > \sqrt{\frac{2 \ln(2 c \tau^2 / \delta_1)}{\tau}}
\end{align*}
for any two users $i, i'$ with different frequencies $p_{i} \neq p_{i'}$, when $\tau \ge \frac{16}{\gamma_p^2}\ln\frac{512 c n_u}{\gamma_p^4 \delta_1}$, or in any phase $s \ge \ceil{\log_2 \left(\frac{16}{\gamma_p^2}\ln\frac{512 c n_u}{\gamma_p^4 \delta_1} \right)}$.
Then a user will be split out with high probability after the first appearance of next stage if her running cluster contains other users of different frequencies. By \cite[Lemma 8]{li2019online}, with probability at least $1 - \delta_1 - \delta_2$, for any user $i$, her cluster only contains the users of the same frequency when
\begin{align*}
\tau \ge \frac{32}{\gamma_p^2}\ln\frac{512 c n_u}{\gamma_p^4 \delta_1} + \frac{16}{p_i}\ln\frac{n_u T}{\delta_2} + \frac{4}{p_i} =: \tau_1(p_i, \delta_1, \delta_2)\,.
\end{align*}

Fix user $i$. By Lemma \ref{lem:theta-hat-good-when-T}, with probability at least $1 - \delta_3$, for any user $i$,
\begin{align*}
\norm{\hat{\theta}_{i,\tau} - \theta_i} \le \frac{R \sqrt{d\ln(1 + \frac{T_{i, \tau}}{d}) + 2\ln\frac{2}{\delta_3} } + \sqrt{A(\delta_3/(2 n_u))}}{\sqrt{T_{i, \tau}\lambda_x/8}} < \frac{\gamma}{4}
\end{align*}
when 
\begin{align*}
T_{i, \tau} &\ge B(\delta_3) = \frac{8192}{\gamma^2 \lambda_x^3} \ln \frac{1024 d}{\lambda_x^2 \delta_3} \,.
\end{align*}
By \cite[Lemma 8]{li2019online}, the condition on $T_{i, \tau}$ holds when $\tau \ge \frac{16}{p_i}\ln\frac{T}{\delta_4} + \frac{4}{p_i}B(\delta_3) =: \tau_2(p_i, \delta_3, \delta_4)$ with probability at least $1 - \delta_4$.
Replace $\delta_3, \delta_4$ by $\delta_3 / n_u, \delta_4 / n_u$. Then with probability at least $1-\delta_3-\delta_4$,
\begin{align*}
\norm{\hat{\theta}_{i,\tau} - \hat{\theta}_{i',\tau}} \ge \frac{\gamma}{2}
\end{align*}
for any two users $i,i'$ with the same frequency $p$ but different weight vectors when $\tau \ge \tau_2(p, \delta_3, \delta_4)$.

Thus with probability at least $1-\sum_{k=1}^4 \delta_k$, for any user $i$, her cluster $j$ only contains users with the same frequencies when $\tau \ge \tau_1(p_i, \delta_1, \delta_2)$ and 
\begin{align*}
\norm{\hat{\theta}_{\tau}^j - \hat{\theta}_{i',\tau}} < \frac{\gamma}{4}
\end{align*}
for at most one true sub-cluster $V \ni i'$ in cluster $j$. By \cite[Lemma 8]{li2019online}, with probability at least $1 - \delta_5$, all users with frequency $p$ will appear at least once after
\begin{align*}
\frac{16}{p}\ln\frac{n_u T}{\delta_5} + \frac{4}{p}
\end{align*}
rounds. Therefore, we have proved the following claim.
\begin{claim}
With probability at least $1 - \sum_{k=1}^5 \delta_k$, for each user $i$, her cluster only contains the users of the same frequency and same prediction vector when
\begin{align*}
\tau \ge 2 \max\left\{\tau_1(p_i, \delta_1, \delta_2), \quad \tau_2(p_i, \delta_3, \delta_4)\right\} + \frac{16}{p_i}\ln\frac{n_u m T}{\delta_5} + \frac{4}{p_i}\,.
\end{align*}
\end{claim}

Let $\delta_1 = \ldots = \delta_5 = \delta / 10$. The requirement for $\tau$ would be
\begin{align*}
\tau \ge O \left(\frac{1}{\gamma_p^2}\ln\frac{n_u}{\gamma_p \delta} + \frac{1}{p_i} \ln\frac{n_u T}{\delta} + \frac{1}{p_i \gamma^2 \lambda_x^3} \ln\frac{d}{\lambda_x \delta} + \frac{1}{p_i}\ln\frac{n_u mT}{\delta} \right) =: T_0(i) \,.
\end{align*}
Note that in the first $T_0(i)$ rounds, the frequency of user $i$ is $p_i$, thus will cancel the $p_i$ in the denominator. So far we have proved that our SCLUB will split well when a running cluster contains users of different true clusters. 

Then we discuss the performance of SCLUB when running clusters are subsets of true clusters. Since $\abs{\frac{T_1}{n_1 T} - p} < a$ and $\abs{\frac{T_2}{n_2 T}- p} < a$ can derive $\abs{\frac{T_1+T_2}{(n_1+n_2)T} - p} < a$, the merged set of two running clusters with the same frequency probability will keep the frequency accuracy; for a running cluster which only contains users of frequency probability $p_i$ and $\tau \ge T_0(i)$, then the distance between the averaged frequency of the running cluster and any of the user inside will be less than $\gamma_p / 2$, thus the running cluster will not be split up. For a good running cluster (which is a subset of true cluster), the distance of its estimated weight vector with the true weight vector will be at most $\gamma/4$ when $T^j$ is large enough with high probability like Lemma \ref{lem:theta-hat-good-when-T}, then the distance between the estimated weight vector of the cluster and the user inside will be at most $\gamma/2$, thus the cluster will not be split up. For two good running clusters with the same weight vector, the distance of their estimated weight vector will be at most $\gamma/2$ with high probability, thus they will be merged. Since there are at most $n_u$'s merge operation and each merge operation will produce a new subset of users, the high-probability property only needs to be guaranteed by at most $n_u$'s true sub-clusters, thus the requirement for $\tau \ge T_0(i)$ will be doubled to guarantee the operations on good running clusters with frequency probability $p_i$ are good.

Then by the regret bound of linear bandits \cite[Theorem 3]{abbasi2011improved} and taking $\delta = \frac{1}{\sqrt{T}}$, the cumulative regret satisfies
\begin{align*}
R(T) &\le \sum_{j=1}^m 4 \beta \sqrt{d p^j T \ln(T/d)}  + O\left( \left(\frac{1}{\gamma_p^2} + \frac{n_u}{\gamma^2 \lambda_x^3}\right) \ln(T) \right)\\
  &= O(d\sqrt{m T} \ln(T))
\end{align*}
where $\beta \ge R \sqrt{d \ln(1 + T / d) + 2 \ln(4 m n_u)}$.
\end{proof}
	
\section{Proofs in the Discussion Part}

\begin{proof}[of \cref{thm:club}]
With the help of the analysis in Lemma \ref{lem:theta-hat-good-when-T}, we are ready to prove that \cref{thm:club}.

By \cite[Lemma 8]{li2019online}, with probability at least $1 - \delta$, $\norm{\hat{\theta}_{i,\tau} - \theta_i} \le \frac{\gamma}{4}$ holds for all users when
\begin{align*}
\tau &\ge \frac{16}{p_{\min}}\ln\frac{2 n_u T}{\delta} + \frac{4}{p_{\min} } B\left(\frac{\delta}{2 n_u}\right) = O\left( \frac{1}{p_{\min}} \ln\frac{T}{\delta} + \frac{1}{p_{\min} \gamma^2 \lambda_x^3} \ln\frac{d n_u}{\lambda_x \delta} \right)\,.
\end{align*}
Then the estimators of weight vectors are accurate enough for all users. Thus the edges between users of different clusters will be deleted by this time and the true clustering will be formed.

Then by the regret bound of linear bandits \cite[Theorem 3]{abbasi2011improved} and taking $\delta = \frac{1}{\sqrt{T}}$, the cumulative regret satisfies
\begin{align*}
R(T) &\le \sum_{j=1}^m 4 \beta \sqrt{d p^j T \ln(T/d)} + O\left(\frac{1}{p_{\min} \gamma^2 \lambda_x^3} \ln(T) \right)\\
  &= O(d \sqrt{m T} \ln(T))\,.
\end{align*}
\end{proof}

\begin{proof}[of \cref{thm:multiple users}]
The key challenge in this theorem is that the number of users per round can be as large as the total number of users $n_u$. The result for the self-normalized series with a fixed number of updates is not satisfying (see \cite[Lemma 6]{li2019online}) since then the regret would have a term of $\sqrt{n_u}$ before $\sqrt{T}$ instead of $\sqrt{m}$. However a good property is that the expected number of users per round can be well controlled, which leads to a good bound for the self-normalized series. We state the result in Lemma \ref{lem:selfNorm}.
\end{proof}

\section{Technical Lemmas}

\begin{lemma}[Hoeffding's Inequality \cite{hoeffding1963probability}]
\label{lem:additive chernoff bound}
Let $X_1, \ldots, X_n$ be independent random variable with common support $[0, 1]$. Let $\bar{X} = \frac{1}{n} \sum_{i=1}^n X_i$ and $\EE{\bar{X}} = \mu$. Then for all $a \ge 0$,
\begin{align*}
&\PP{\bar{X} - \mu \ge a} \le  \exp(- 2 n a^2),\quad \PP{\bar{X} - \mu \le -a} \le  \exp(- 2 n a^2)\,.
\end{align*}
\end{lemma}


\begin{lemma} \label{lem:T-for-ratio-less-than-gamma}
\begin{align*}
\frac{R \sqrt{d \ln(1 + \frac{T}{d}) + 2\ln\frac{1}{\delta_2} } + \sqrt{A(\delta_1)}}{\sqrt{T\lambda_x/8}} < \frac{\gamma}{4}  
\end{align*}
is satisfied when
\begin{align*}
T > \frac{8}{\gamma^2 \lambda_x}\max\left\{A(\delta_1), \quad 4R^2d\log\frac{32R^2}{\gamma^2 \lambda_x}, \quad 4R^2\ln\frac{1}{\delta_2} \right\}\,.
\end{align*}
\end{lemma}
\begin{proof}
To prove this, it is enough to prove $\frac{\sqrt{A(\delta_1)}}{\sqrt{T\lambda_x/8}} \le \frac{\gamma}{8}$ and $\frac{R \sqrt{d \ln(1 + \frac{T}{d}) + 2\ln\frac{1}{\delta_2} }}{\sqrt{T\lambda_x/8}} < \frac{\gamma}{8}$.

The first condition is equivalent to $T \ge \frac{8 A(\delta_1)}{\gamma^2 \lambda_x}$. 

The second condition can be derived by $(a) \frac{d\ln(1+T/d)}{T\lambda_x} < \frac{\gamma^2}{16 R^2}$ and $(b)\frac{2\ln\frac{1}{\delta_2}}{T\lambda_x} < \frac{\gamma^2}{16 R^2}$. (a) is satisfied when $T > \frac{32R^2d}{\gamma^2 \lambda_x}\log\frac{32R^2}{\gamma^2 \lambda_x}$ by \cite[Lemma 9]{li2019online}. (b) is equivalent to $T > \frac{32R^2}{\gamma^2 \lambda_x}\ln\frac{1}{\delta_2}$.

Thus summarizing the three conditions for $T$ to finish the proof.
\end{proof}

\begin{lemma} \label{lem:selfNorm}
Let $M_n = M + \sum_{t=1}^n \sum_{k=1}^{K_t} x_{t,k} x_{t,k}^{\top}$, where $M \in \RR^{d \times d}$ is a strictly positive definite matrix and $x_{t,k} \in \RR^d$ is a $d$-dimensional column vector. If $\sum_{k=1}^{K_t} \norm{x_{t,k}}_{M_{t-1}^{-1}}^2 \leq 1$, then
\begin{equation}
\sum_{t=1}^n \sum_{k=1}^{K_t} \norm{x_{t,k}}_{M_{t-1}^{-1}}^2 \leq 2\log\frac{\det(M_n)}{\det(M)}\,.
\end{equation}
Furthermore, if $\norm{x_{t,k}}_2 \leq L, K_t \le A, \EE{K_t} = K, M = \lambda I, \lambda \geq A L^2, \forall t,k$, then
\begin{equation}
\label{eq:selfNormSum}
\sum_{t=1}^n \sum_{k=1}^{K_t} \norm{x_{t,k}}_{M_{t-1}^{-1}} \leq \sqrt{2dnK \log\left(1 + \frac{n A L^2}{\lambda d}\right)}.
\end{equation}
\end{lemma}
\begin{proof}
\begin{align*}
&\det(M_n)= \det(M_{n-1}) \det\left(I + M_{n-1}^{-1/2} \left(\sum_{k=1}^{K_n}x_{n,k} x_{n,k}^{\top} \right) M_{n-1}^{-1/2} \right)\\
\overset{(a)}{\ge}& \det(M_{n-1}) \left(1 + \sum_{k=1}^{K_n} \norm{x_{t,k}}_{M_{n-1}^{-1}}^2 \right) \ge \det(M) \prod_{t=1}^n \left(1 + \sum_{k=1}^{K_t} \norm{x_{t,k}}_{M_{t-1}^{-1}}^2 \right),
\end{align*}
where (a) is by Lemma A.3 of \cite{li2016contextual}. Then
\begin{align*}
\sum_{t=1}^n \sum_{k=1}^{K_t} \norm{x_{t,k}}_{M_{t-1}^{-1}}^2 \overset{(b)}{\le} \sum_{t=1}^n 2 \log\left(1 + \sum_{k=1}^{K_t} \norm{x_{t,k}}_{M_{t-1}^{-1}}^2 \right) \le 2\log\frac{\det(M_n)}{\det(M)}\,,
\end{align*}
in which (b) is due to $2\log(1+u) \geq u$ for $u\in[0,1]$.

If $\norm{x_{t,k}}_2 \leq L, K_t \le A, \EE{K_t} = K, M = \lambda I, \lambda \geq A L^2, \forall t,k$, then
\begin{align*}
&\EE{\sum_{t=1}^n \sum_{k=1}^{K_t} \norm{x_{t,k}}_{M_{t-1}^{-1}}} \le \left( \EE{\sum_{t=1}^n \sum_{k=1}^{K_t} \norm{x_{t,k}}_{M_{t-1}^{-1}} }^2 \right)^{1/2} \le \left( \EE{\sum_{t=1}^n K_t \sum_{k=1}^{K_t} \norm{x_{t,k}}_{M_{t-1}^{-1}}^2} \right)^{1/2} \\
\le &\left( \EE{\sum_{t=1}^n K_t \cdot 2 \log\frac{\det(M_n)}{\det(\lambda I)}} \right)^{1/2} \le \sqrt{2dnK \log\left(1 + \frac{n A L^2}{\lambda d}\right)}.
\end{align*}
\end{proof}

\section{More Discussions}
\label{sec:more discussions}

The assumptions on the gap parameters $\gamma_{\theta}, \gamma_{p}$ are trade-off between personalization and collaborative filtering. If the gap parameters are large enough, there is only one cluster containing all users; if the gap parameters are small enough, each user is one cluster. Neither case would perform well, because the first case completely ignores the individual preferences, and the second case does not exploit the similarity of users and the potential collaborative filtering advantage and recommends based only on each user’s very limited feedback. These parameters are for the assumption on perfect underlying clustering structure. The confidence radius of the estimates of weight vectors and frequencies is controlled by the $F$ function, which is decreasing as more data flows in. Thus as time goes by, the algorithm can adaptively find finer and finer clustering over users and each running clustering is good in the corresponding accuracy level. If there is not a perfect clustering structure, then by a similar analysis, we could derive an asymptotic regret bound to characterize the adaptive clustering behaviors. We omit this part and simply assume the gap exists.
\section{Complexity of Implementation}
\label{sec:complexity of implementation}

\paragraph{\sclub} Each recommendation takes $O(L d^2)$ time where the matrix inverse is updated by the Sherman–Morrison formula. The update after receiving feedback takes $O(d^2)$ time. The time for each split check and each merge check is $O(d)$. Due to the decreasing property of function $F$, the merge check will only perform on the cluster containing current user and the merge will continue until no merge is possible.

For the rounds where the current clustering structure is true, \sclub only needs to check split and merge but does not need to split and merge. The number of split check is $1$ and the number of merge check is $m$. So the time complexity for this part is $O(L d^2 + m d)$ per round.

For the rounds on exploring clustering structure, the time for each split is $O(d^2)$ and the time for each cluster merge is $O(d^3)$ where the matrix inverse and estimated weight vector for the new cluster need to recompute. Note the number of merge times is at most $n_u$. So the time for this part is $O(L d ^2 + n_u d^3)$. 

By the upper bound for rounds of exploration on clustering structure, the time complexity (in expectation) for $T$ rounds is 
\begin{align*}
O\left(TLd^2 + Tmd + \left(\frac{n_u d^3}{\gamma_p^2} + \frac{n_u^2 d^3}{\gamma^2 \lambda_x^3}\right) \ln(T) \right)\,.
\end{align*}

\paragraph{\club} The time for recommendation, update and check edge removal are the same. For the rounds of true clustering, CLUB only needs to check edge removal but does not need to cut edges. The number of such checks is (in expectation) $O\left(\sum_{j=1}^m p^j n^j\right)$. So the time for this part is $O\left(L d^2 + d\sum_{j=1}^m p^j n^j \right)$. For the rounds of exploring clustering structure, the total number of checks on edge removal (in expectation) is
\begin{align*}
O\left(\sum_{i=1}^{n_u} p_i \sum_{(i',i) \in E_1} \frac{R}{\min\{p_i, p_{i'}\} \gamma \lambda_x^2} \ln(T)\right)\,,
\end{align*}
where $E_1$ is the edge set after randomly initializing the graph. Once there is an edge removal, the time for recomputing connected component is $O(n_u \ln^{2.5} (n_u))$. The time for recomputing the estimated weight vectors is $O(m(d^3 + n_u d^2))$. Overall, the time complexity (in expectation) for $T$ rounds is
\begin{align*}
&O\left( TLd^2 + Td\sum_{j=1}^{m}p^j n^j + \frac{n_u d R}{p_{\min} \gamma \lambda_x^2} \ln(T) + n_u \abs{E_1} \log^{2.5}(n_u)  + m(n_u d^2 + d^3) \right)\,.
\end{align*}
Note that $\sum_{j=1}^{m} p^j n^j$ is usually larger than $m$.


The time complexities for both \linucbone and \linucbind are $O(TLd^2)$, where the recommendation and updates are mainly performed.

\section{Quantity Results for Experiments}
\label{sec:quantity results for experiments}

The regrets of \cref{fig:results} at the end are given in the \cref{table:regrets}, while total running times (wall-clock time) are shown in \cref{table:running time}. The experiments are run on Dell PowerEdge R920 with CPU of Quad Intel Xeon CPU E7-4830 v2 (Ten-core 2.20GHz) and memory of 512GB.

\begin{table*}[thb!]
\centering
\begin{tabular} {l|r|r|r|r|r|r|r}
Regret & (1a) & (1b) &(1c)  &(2a)  &(2b)  &(2c)  &(2d)  \\ 
\hline
\club& $72,546$ &$62,854$ &$67,116$ &$72,733$ &$67,053$ &$72,803$ &$67,887$\\
\linucbind &$72,481$ &$58,192$ &$60,805$ &$72,750$ &$61,042$ &$72,577$ &$60,804$\\
\linucbone &$151,519$ &$118,157$ &$151,097$ &$169,018$  &$168,796$ &$171,820$, &$171,199$\\
\sclub &$68,238$ &$54,672$ &$57,255$ &$68,413$ &$57,105$ &$68,482$ &$56,863$\\
\hline
\end{tabular}
\caption{The regrets at the end of running in \cref{fig:results}. The values are averaged over the $10$ random runs.}
\label{table:regrets}
\end{table*}

\begin{table*}[thb!]
\centering
\begin{tabular} {l|r|r|r|r|r|r|r}
Time (s) & (1a) & (1b) &(1c)  &(2a)  &(2b)  &(2c)  &(2d)\\ 
\hline
\club& $1,004$ &$1,040$ &$2,020$ &$1,145$  &$1,134$ &$1,357$ &$1,925$\\
\linucbind &$771$ &$797$ &$1,111$ &$1,752$  &$1,497$ &$995$ &$2,705$\\
\linucbone &$741$ &$768$ &$1,306$ &$1,327$  &$1,603$ &$993$ &$2,422$\\
\sclub &$4,284$ &$4,272$ &$4,855$ &$12,170$  &$6,008$ &$4,880$ &$8,196$\\
\hline
\end{tabular}
\caption{The total running time of compared algorithms in \cref{fig:results} in the unit of seconds(s). The values are averaged over the $10$ random runs.}
\label{table:running time}
\end{table*}

\fi

\end{document}